\documentclass{article}

 \usepackage[preprint]{neurips_2026}

\usepackage[utf8]{inputenc} 
\usepackage[T1]{fontenc}    
\usepackage{hyperref}       
\usepackage{url}            
\usepackage{booktabs}       
\usepackage{amsfonts}       
\usepackage{nicefrac}       
\usepackage{microtype}      
\usepackage{xcolor}         
\usepackage{mathtools}      
\usepackage{titletoc}

\title{Exemplar-Derived Soft Prompts for Efficient and Effective Domain Adaptation}

\author{%
  Abhinav Jain\thanks{Equal contribution.} \\
  Department of Computer Science\\
  Rice University\\
  \texttt{aj70@rice.edu} \\
  \And
  Xinyu Yao\textsuperscript{*} \\
  Department of Computer Science\\
  Rice University\\
  \texttt{xy38@rice.edu} \\
  \AND
  Thomas Reps \\
  Department of Computer Science\\
  University of Wisconsin--Madison\\
  \texttt{reps@cs.wisc.edu} \\
  \And
  Christopher Jermaine \\
  Department of Computer Science\\
  Rice University\\
  \texttt{cmj4@rice.edu} \\
}

\usepackage{comment}
\usepackage{xspace}
\usepackage{graphicx}
\usepackage{subcaption}
\usepackage{overpic}
\usepackage{svg}
\usepackage{bm}
\usepackage[hang,flushmargin]{footmisc}
\usepackage{tabularx}
\usepackage{wrapfig}

\usepackage{multicol}
\usepackage{multirow}
\usepackage{colortbl}
\usepackage{float}
\usepackage{placeins}
\usepackage{adjustbox}
\usepackage{array}
\newcolumntype{P}[1]{>{\centering\arraybackslash}p{#1}}  

\usepackage{etoolbox}
\makeatletter
\patchcmd{\@makecaption}
{\scshape}
{}
{}
{}
\newcommand{\onetagright}{\tagsleft@false}
\makeatother

\usepackage{amsmath,amsfonts,bm}
\usepackage{amsthm,amssymb}

\theoremstyle{plain}

\theoremstyle{definition}

\theoremstyle{remark}

\def\eqref#1{equation~\ref{#1}}

\def\1{\bm{1}}

\def\vk{{\bm{k}}}

\def\vq{{\bm{q}}}

\def\vv{{\bm{v}}}

\def\vz{{\bm{z}}}

\def\mK{{\bm{K}}}

\def\mV{{\bm{V}}}
\def\mW{{\bm{W}}}

\def\mZ{{\bm{Z}}}

\DeclareMathAlphabet{\mathsfit}{\encodingdefault}{\sfdefault}{m}{sl}
\SetMathAlphabet{\mathsfit}{bold}{\encodingdefault}{\sfdefault}{bx}{n}

\def\sD{{\mathbb{D}}}

\newcommand{\softmax}{\mathrm{softmax}}

\DeclareMathOperator*{\argmax}{arg\,max}

\usepackage{algorithm}
\usepackage{algpseudocode}

\newcommand{\approach}{\textsc{MHA-ESP}\;}
\newcommand{\approachB}{(\textsc{MHA-ESP})}
\newcommand{\xrag}{\textsc{U-ESP}\;}
\newcommand{\xragk}{\textsc{P-ESP}\;}
\newcommand{\pt}{\textsc{PT}\;}
\newcommand{\idpg}{\textsc{IDPG}\;}

\begin{document}

\maketitle

\begin{abstract}
Adapting foundation models to new domains is challenging and computationally expensive. While parameter-efficient fine-tuning (PEFT) methods allow models to acquire domain-specific skills, they require updating deployed models and incur additional training and deployment overhead. In contrast, In-context Learning (ICL) avoids model updates and improves over off-the-shelf models by leveraging similar exemplars, but it often fails to achieve competitive accuracy on specialised tasks. Motivated by the success of in-context exemplars and the need for fine-tuning-level adaptation, we propose Multi-Head Attention-based Exemplar Soft Prompting \approachB, which uses an attention mechanism to learn soft prompts from retrieved exemplars, with multiple attention heads controlling prompt generation. Across multiple benchmarks and model scales, \approach performs on par with Low-Rank Adaptation (LoRA), outperforms standard ICL by an average of 18.85 points, and reduces inference cost by up to 10$\times$GFLOPs, enabling efficient, high-accuracy domain adaptation without updating the foundation model.
\end{abstract}

\section{Introduction}
\label{sec:introduction}
Fine-tuning is the standard approach for adapting foundation models to a domain by acquiring the skills needed to solve a specific task. In practice, foundation models are often adapted via Parameter-Efficient Fine-Tuning (PEFT), such as LoRA \cite{hu2022lora}, which updates only a small subset of the model’s weights to save training cost while preserving performance \cite{han2024parameter}. However, both full fine-tuning and PEFT have two key drawbacks. First, they require modifying the foundation model itself: in multi-tenant settings where a single backbone must serve users with distinct tasks (Tasks A, B, and C in Figure~\ref{fig:mha-esp-intro}), the server is forced to maintain and route to a separate adapter for every user-task pair \cite{wen2023batched}. Second, the upfront training cost remains high: adapting a massive model demands substantial computation resources and data, an expense that grows with model size \cite{hu2022lora}. 

As an alternative, In-Context Learning (ICL) has emerged as a training-free  ``plug-and-play'' adaptation strategy \cite{ xie2021explanation, min2022rethinking}. Rather than updating model parameters through gradient descent, ICL conditions the model at inference time by providing a small set of task-specific exemplars directly in the input prompt, as shown in Figure \ref{fig:mha-esp-intro}.  Pioneering work such as \cite{brown2020language, wei2022chain} demonstrated that models exhibit strong in-context learning when presented with exemplars in natural-language form. Remarkably, in certain settings, ICL has even outperformed fine-tuning-based methods  \cite{mallen2022not,mosbach2023few, pingua2025medical}.

However, ICL does not excel in all tasks. It performs poorly in tasks with out-of-distribution data where the model must learn a new skill \cite{yu2024evaluation, gupta2024comprehensive}. We found this to be a core limitation of ICL in our experiments as well: simply providing a few exemplars in context as text may not allow the model to acquire the required skill, especially when the skill requires reasoning beyond the data in its original training distribution.


Our goal is to develop a new exemplar-based method that has the advantages of ICL (no model modification, low upfront cost) but also performs well on skills or tasks for which the model was not trained. We investigate a new method called \emph{Multi-Head Attention Exemplar Soft Prompt} \approachB. The idea behind it is to efficiently translate the exemplars into a soft-prompt-based representation, a format the frozen foundation model can more readily exploit. By paying a small upfront training cost, it learns an auxiliary model that produces such task-conditioned embeddings and prepends them to the query.   
The approach preserves the deployment simplicity of ICL —no extra server-side parameters need to be maintained, as demonstrated in Figure \ref{fig:mha-esp-intro}, since the soft prompt is built before it is sent to the foundation model—and introduces only a lightweight training overhead for the soft-prompt encoder on the user side.  

A key insight from our findings is that learning soft prompts from in-context exemplars is crucial to effective adaptation. In other words, simply learning a static soft prompt \cite{liu2021p,lester2021power,liu2024gpt} for a new task is not enough, which we find incurs a noticeable performance degradation.  Across a wide range of foundation models, specialised tasks (e.g., molecular design, captioning, QA), and domains (e.g., chemistry, economics, math), our experiments show that the proposed exemplar-conditioned prompting consistently outperforms both the standard text-based retrieval and prompt-tuning baselines and closes the gap with fine-tuning baselines like LoRA. Furthermore, \approach's design provides (i) gains in inference efficiency by reducing quadratic token overhead, (ii) stable and consistent results with the property of exemplar order-invariance embodied in the design.

\begin{figure*}[t]
    \centering
    \includegraphics[width=1.0\linewidth]{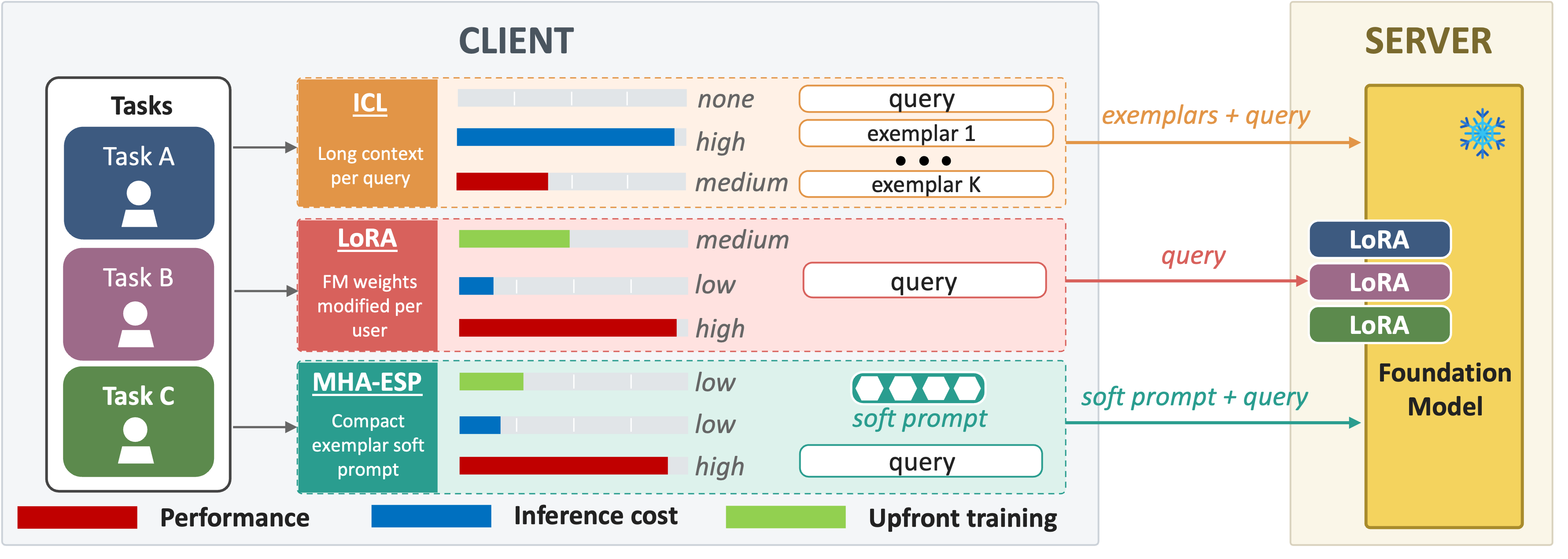}
    \caption{Comparison of domain adaptation strategies. ICL sends exemplars with query, incurring a high inference cost. LoRA modifies foundation model weights and requires per-user adapters on the SERVER side. \approach (ours) encodes exemplars into a compact soft prompt, achieving high performance with low inference cost while keeping trainable parameters locally on the USER side.}
    \label{fig:mha-esp-intro}
\end{figure*}

In this work, we make the following contributions:
\begin{itemize}
    \item We address fundamental limitations in ICL, which include unsatisfactory performance, high inference cost, and exemplar-order variance in domain adaptations.
    \item We propose an exemplar-order-invariant multi-head-attention architecture for soft-prompt generation in \textsc{MHA-ESP}, controlled by varying the attention heads.
    \item Our experiments show on-par performance with LoRA and an average 19-point gain over ICL across benchmarks and models, while cutting inference cost by $10\times$ in GFLOPs.
\end{itemize}

\section{Related Works}
\label{sec:related_work}
\textbf{Domain Adaptation via In-Context Learning.} In-Context Learning (ICL) has emerged as the most widely used training-free approach for adapting foundation models to new domains. Instead of updating model parameters, ICL conditions the model at inference time on a small set of exemplars that demonstrate the target domain \cite{brown2020language,xie2021explanation,min2022rethinking}. This paradigm has demonstrated strong effectiveness across settings: in chain-of-thought (CoT) prompting where exemplars annotated with intermediate reasoning steps elicit multi-step reasoning behaviour \cite{wei2022chain,kojima2022zerocot}; in retrieval-augmented generation for grounding in external knowledge \cite{lewis2020retrieval,guu2020retrieval, borgeaud2022improving}; and in agent-based frameworks \cite{yao2023react, shinn2023reflexion, schick2023toolformer}. Across these settings, ICL consistently improves over zero-shot prompting without modifying model weights.

Despite its advantages, ICL faces three fundamental limitations. First, because attention scales quadratically with sequence length \cite{vaswani2017attention}, appending more exemplars or retrieving context into the prompt incurs substantial inference cost. Second, ICL exhibits exemplar-order variance; permuting the input demonstrations can shift accuracy by more than 10 points on standard classification benchmarks \cite{lu2022fantastically,pezeshkpour2023large,zheng2023large}. Third, ICL still underperforms on out-of-distribution or skill-acquisition domains where exemplars themselves failed to surface semantically sufficient context \cite{yu2024evaluation, gupta2024comprehensive,joren2024sufficient}. In contrast, \approach with a small upfront training cost retains the adaptation signal of in-context exemplars but encodes them into a small, order-invariant set of soft-prompt vectors, reducing effective context length by an order of magnitude while delivering better performance than text-based ICL.

\textbf{Domain Adaptation via Fine-Tuning.} When training data and computational resources are available, fine-tuning typically yields the strongest domain-adaptation performance. However, full fine-tuning is prohibitive in both memory and compute \cite{hu2022lora}. Parameter-efficient fine-tuning methods have emerged to address this by freezing the backbone and updating a subset of parameters, generally referred to as the adapter modules \cite{houlsby2019adapter,pfeiffer2021adapterfusion, li2021prefix, he2021towards, hu2022lora, zhang2023adalora}. Among these, LoRA \cite{hu2022lora} has become the standard for domain adaptation.

However, a practical limitation is that the adaptation resides \textit{on the server side}: the serving infrastructure must store and load a separate adapter for each domain, leading to significant memory and management overhead when scaling to many heterogeneous domains. Recent work has therefore focused on multi-tenant LoRA serving: FLoRA \cite{wen2024flora} lets each example in a minibatch carry its own low-rank weights; S-LoRA \cite{sheng2024slora} servers thousands of concurrent adapters via unified paging; dLoRA \cite{wu2024dlora} dynamically merges/unmerges adapters with the base model and migrates requests across replicas. These systems reduce serving costs but do not address the underlying constraint of the foundation model provider hosting the adapters. \approach addresses this by shifting adaptation to the user side, where each domain-specific request trains a lightweight attention-based module (a few million parameters) to generate soft prompts. The resulting soft prompt is prepended to each query at inference time, eliminating the need for maintaining per-domain adapters.

\textbf{Domain Adaptation via Soft Prompting.} Soft-prompt methods occupy a middle ground: like PEFT, they have a small training overhead, but like ICL, they operate purely at the input level, leaving model weights and the serving stack untouched. Prompt tuning \cite{lester2021power} and P-Tuning \cite{liu2021p} showed that prepending a few trainable vectors suffices to adapt very large models, making them well-suited for the multi-tenant setting \cite{jain2024prompt}. Prior work either learns a static prompt per domain \cite{lester2021power, liu2024gpt}, which lacks instance-specificity, or conditions soft prompts on individual queries \cite{wu2022idpg, jain2024prompt}.

However, even these adaptive methods overlook the rich signal provided by similar exemplars—the key driver behind the success of ICL. \approach is explicitly designed to recover this signal directly: it retrieves exemplars relevant to the current query and learns to encode them into a query-specific soft prompt. By leveraging multi-head attention, \approach enables structured interactions between the query and retrieved exemplars, while maintaining order invariance in exemplar aggregation. In addition, the number of attention heads provides a simple and interpretable hyperparameter that controls the expressiveness of the resulting soft prompts, allowing the method to be flexibly applied across tasks and foundation models.
\section{Methodology}
\label{sec:methodology}
The objective of domain adaptation is to adapt a foundation model $f_{\theta}$ to a new task with data $\sD=\{(x_i,y_i)\}_{i=1}^N$ such that $y=f_\theta(x)$. It formally involves updating the parameters of the model $\theta$ to maximize the likelihood of the response $y$, i.e., $\max_\theta \mathbb{E}_{(x,y)\sim \sD}[\log p_\theta(y|x)]$. 

\textbf{Domain adaptation with in-context exemplars.} 
In domains where parametric updates are not feasible, either due to limited training data or computationally expensive training, adaptation can be achieved by providing a small set of domain-specific samples $\{(x_k, y_k)\}_{k=1}^K\subset \sD$ as in-context exemplars. In this setting, the model prediction takes the form $y=f_\theta(\{(x_k, y_k)\}_{k=1}^K, x)$. This adaptation can be further enhanced by leveraging a relevance or similarity function $\Phi(x, x')$. Specifically, for each input $x$, the top-$K$ most relevant samples in $\sD$ are retrieved as in-context exemplars---i.e., $(x_k, y_k) = \argmax_{(x', y')\in \sD\setminus\{(x,y), (x_1,y_1),\ldots,(x_{k-1},y_{k-1})\}} \Phi(x, x')$, for $1 \leq k \leq K$.

\textbf{Representing in-context exemplars.} While approaches based on ICL typically represent exemplars in text form, we instead consider soft prompts as an alternative representation. To this end, we define an encoding function $g(\sD_K|x)=\mZ$ that embeds $\sD_K$, the top-$K$ in-domain samples retrieved for a given $x$, into a soft prompt $\mZ\in\mathbb{R}^{d\times m}$, where $m$ denotes this virtual prompt's length and $d$ is the embedding dimension of the foundation model. Similar to prompt tuning, the learned soft prompt is prepended to the word embeddings of $x$, thereby enabling adaptation without updating $\theta$.

\subsection{Multi-Headed Attention Exemplar Soft Prompt}
\label{subsec:proposed formulation}
In this paper, we propose a multi-headed attention-based function for encoding exemplars into soft prompts \approachB, where each head projects the top-$K$ in-context exemplars into the embedding space of the language model as a single soft-prompt vector $\vz\in\mathcal{R}^{d}$. In particular, performing attention with $H$ heads results in the soft prompt $\mZ_{\textit{MHA}}=[\vz^{(1)},\ldots \vz^{(i)}\ldots,\vz^{(H)}]$ of length $m=H$, where the output from the $i^{\textit{th}}$ head is
\begin{align*}
    \vz^{(i)} &= \textsc{Attention}(\vq^{i}, \textbf{K}^{i}, \textbf{V}^{i}) = \softmax(\frac{\vq^{i}\mK^{i}}{\sqrt{d}})\mV^{i};
\end{align*}
with $\mK^{i}=[\vk^{i}_1,..\vk^{i}_k..,\vk^{i}_K] \; \text{and} \; \mV^{i}=[\vv^{i}_1,..\vv^{i}_k..,\vv^{i}_K]$, where $\vq^{i}=\mW^{Q}_{i}E_x$ is the attention query corresponding to a given $x$, $\vk^{i}_k=\mW^{K}_{i}E_{x_k\oplus y_k}$ and $\vv^{i}_k=\mW^{V}_{i}E_{x_k\oplus y_k}$ are the keys and values corresponding to $(x_k,y_k)\in\sD_K$, respectively. $E_x\in\mathcal{R}^{d'}$ represents the dense representation of $x$ obtained by a sentence-embedding model with hidden dimension $d'$;
$\oplus$ denotes the concatenation operator; and
$\mW^{Q}_{i}\in\mathcal{R}^{d\times d'}$, $\mW^{K}_{i}\in\mathcal{R}^{d\times d'}$, and $\mW^{V}_{i}\in\mathcal{R}^{d\times d'}$ represent the attention query, key, and value weights associated with head $i$, respectively.
Overall, we optimize the following objective:
\begin{equation*}
    \max_{\varphi}\mathbb{E}_{(x,y)\sim \sD, \, \sD_K=\Phi(x,\cdot)}\big[\log p(y|g_{\varphi}(\sD_K|x), x)\big],
\end{equation*}
such that $g_{\varphi}(\sD_K|x)=\mZ_{\textit{MHA}}$, where $\varphi$ represents the parameters of the encoding function $g$.

\textbf{Lower inference cost.} Prior work \cite{cheng2024xrag} interprets the encoding function $g(\cdot)$ as a context-compression module achieving a compression ratio of $\frac{|\sD_K|}{m}$. It reduces the tokenised length of the context $\sD_K$ from $|\sD_K|$ to $m$, thereby lowering overall inference cost during deployment. ICL without any compression yields a compression ratio of $1$. In contrast, \approach  achieves a higher and configurable compression ratio of $\frac{|\sD_K|}{H}$ ($m=H$). Configurability comes from tailoring the soft-prompt size to the domain by varying the number of heads to adjust its representational capacity.  

\textbf{Order Invariance.} We argue that the soft prompts generated by \approach are invariant to the order of retrieved exemplars. This property follows from the use of scaled dot-product attention, which depends only on the set of input tokens rather than their order \cite{vaswani2017attention}. Specifically, the inputs are exemplar embeddings $[E_{x_1 \oplus y_1}, \dots, E_{x_K \oplus y_K}]$ without positional encoding; therefore, permuting them does not alter the aggregated representation produced by each head. As a result, the soft prompt remains identical regardless of exemplar ordering. A formal proof is given in Appendix \ref{sec:order_invariance}.
\section{Experiments, Results, and Discussion}
\label{sec:experiment}

\subsection{Experimental Setup}
\label{sec:experimental_setup}
\begin{table*}[h]
    \centering
    \caption{Comparison with ICL and LoRA across benchmarks. Hyperparameters are tuned via sweeps: \approach ($H\in \{1, 2, 4, 8\}$). \underline{Underlined} indicates the best performance.  \textbf{Bold} indicates the best improvement relative to LoRA for a given foundation model. 
    }
    \begin{tabular}{lccccccc}
        \toprule
        \textbf{Methods} & \textit{BACE} & \textit{BBBP} & \textit{ClinTox} & \textit{Design} & \textit{Captioning} & \textit{MMLU-Pro} & $\Delta_{\text{LoRA}}^{\text{avg}}$\\
        \hline
        \rowcolor{gray!20}
        \multicolumn{8}{c}{\centering{\textbf{\textit{Qwen3-4B}}}} \\
        \hline
         LoRA & 12.32 & 42.84 & 81.00 & \underline{81.31} & \underline{31.74}  & \underline{50.68} & - \\
         Off-the-Shelf & 0.0 & 10.88 & 0.0  & 37.75 & 0.82  & 35.22 & -42.18\\
         ICL(fixed) & 53.09 & 33.93 & 9.72 & 71.51  & 5.63 & 36.82 & -22.82\\
         ICL & 75.87 & 69.33 & 44.24 & 71.66 & 5.67 & 38.24 & -4.36\\
         \approach & \underline{76.27} & \underline{82.49} & \underline{96.32} & 73.63 & 23.59 & 45.96 & \textbf{13.46}\\
         \hline
         \rowcolor{gray!20}
         \multicolumn{8}{c}{\centering{\textbf{\textit{Llama3.2-3B-Instruct}}}} \\
         \hline
          LoRA & 0.0 & 54.77 & 93.06  & 80.63 & 31.82 & \underline{42.09} & - \\
         Off-the-Shelf & 0.0 & 0.0 & 0.0 & 17.38 & 0.0  & 25.45 & -56.74\\
         ICL(fixed) & 48.86 & 36.39 & 36.23 & 65.71 & 4.74 & 28.67 & -19.03\\
         ICL & 49.89 & 46.15 & 37.90 & 65.44 & 5.72 & 28.94 & -16.69\\
         \approach & \underline{67.62} & \underline{87.61} & \underline{94.44} & \underline{80.73} & \underline{32.79} & 36.43 & \textbf{13.63}\\
         \hline
         \rowcolor{gray!20}
         \multicolumn{8}{c}{\centering{\textbf{\textit{Mistral-7B-Instruct}}}} \\
         \hline
          LoRA & 24.72 & 85.24 & 78.32 & \underline{85.22} & \underline{36.98} & 26.13 & - \\
         Off-the-Shelf & 18.05 & 0.0 & 0.0 & 48.12 & 0.0 & 25.50 & -52.31\\
         ICL(fixed) & 58.63 & 42.33 & 34.01 & 71.11 & 3.51 & 26.55 & -20.93\\
         ICL & \underline{77.65} &  66.62 & 59.79 & 71.11 & 3.53 & 26.48 & -8.64 \\
         \approach & 76.26 & \underline{87.01} & \underline{94.42} & 78.11 & 34.54 & \underline{35.17} & \textbf{9.98}\\
         \bottomrule
    \end{tabular}
    \label{tab:baseline_domain_adaptation}
\end{table*}
\textbf{Benchmarks.} In this work, we focus on domains or tasks where foundation models exhibit weak zero-shot performance, but can benefit from in-context exemplars. Such exemplars provide cues in the form of analogies to similar samples, domain-specific formats (e.g., SMILES strings for molecules), specialized vocabulary, or templates for step-by-step reasoning. We therefore benchmark across two groups of tasks. The first group involves limited-data molecular-property-prediction tasks: (a) \textit{BACE} [train/test split: 1,413/100]---binary classification of whether a molecule inhibits BACE1; (b) \textit{BBBP} [train/test split: 1,950/100]---prediction of blood–brain barrier penetration (Yes/No); (c) \textit{ClinTox} [train/test split: 1,384/100]---classification of molecules as clinically trial toxic vs. non-toxic \cite{guo2023can}. The second group consists of medium-scale tasks on molecular reasoning and general language understanding: (d) \textit{ChEBI-Design} [train/test split: 26,407/100]---generating new molecules from their descriptions and (e) \textit{Captioning} [train/test split: 26,407/100]---generating text-based description that describes the molecule \cite{edwards2022translation,edwards2021text2mol}. (f) \textit{MMLU-Pro} [train/test split: 6,011/6,019] -- enhanced multi-task language understanding across 14 disciplines \cite{wang2024mmlu}.  Additional performance comparison on the math-reasoning benchmark can be found in Appendix \ref{sec:math_benchmark}.

\textbf{Performance Metrics.} For yes/no classification tasks, we report the \textit{geometric mean} of the True-Positive and True-Negative rates, as it avoids the biases of accuracy or F1 under skewed label distributions. We call this quantity the ``\textit{Effective Accuracy}''. For ChEBI-Design, we report Molecular ACCess System (\textit{MACCS}), which measures similarity over molecular structural fingerprints; for Captioning, we report \textit{BLEU-4} to measure phrase-level precision---critical for chemically accurate descriptors. For MMLU-Pro, we report \textit{accuracy} to measure whether the model selects the correct option. To quantify improvement of a method relative to any given baseline, we compute $\Delta_{\text{baseline}}^{\text{avg}}=\sqrt[\leftroot{-2}\uproot{2}n]{\prod_{i=1}^n(100+o_i^{method}-o_i^{baseline})}  - 100$ where $o_i$ is the score of the method on task $i$, and $n$ is the total number of benchmark tasks. 

\textbf{Models.} We evaluate three families of foundation models ($\theta$) with various sizes: the \textit{Qwen3-4B} \cite{yang2025qwen3}, \textit{Llama-3.2-3B-Instruct}, and \textit{Mistral-7B-Instruct-v0.3} \cite{jiang_mistral_2023}. For smaller FM, refer to Appendix \ref{sec:qwen3_0.6b}. To encode in-context exemplars ($E_x$), we employ domain-specific embedding models like \textit{ChemBERTa2} \cite{ahmad2022chemberta} for property prediction tasks and \textit{Qwen3-Embedding-0.6B}  \cite{zhang2025qwen3} for encoding natural language.

\textbf{Retrieval Functions ($\Phi$).} We assume that a domain-specific retrieval function is provided. For molecular benchmarks, we retrieve the top-$K$ relevant exemplars using Tanimoto Similarity \cite{tanimoto1958elementary, guo2023can}, which computes scaffold-level similarity between SMILES representations. For others, retrieval is performed via cosine similarity. Additional experimental details can be found in Appendix \ref{sec:hyperParamSearch}.

\subsection{Comparison with Domain Adaptation Baselines}
\begin{table*}[t]
    \centering
    \caption{Comparison with  soft-prompt-based baselines across benchmarks. Hyperparameters are tuned via sweeps: Prompt Tuning (virtual tokens $\in \{1, 5, 10\}$), \approach ($H \in \{1, 2, 4, 8\}$). Results are averaged over $3$ random seeds. \underline{Underlined} indicates the best performance. \textbf{Bold} indicates the best average performance across benchmark tasks.}
    \begin{tabular}{lccccccc}
        \toprule
        \textbf{Methods} & \textit{BACE} & \textit{BBBP} & \textit{ClinTox} & \textit{Design} & \textit{Captioning} & \textit{MMLU-Pro} & \textit{Average}\\
        \hline
        \rowcolor{gray!20}
        \multicolumn{8}{c}{\centering{\textbf{\textit{Qwen3-4B}}}} \\
        \hline
         PT & 0.0 & 11.10 & 0.0 & 39.96 & 1.13 & 35.27 & 14.58 \\
         IDPG & 62.54 & \underline{87.14} & 94.36 & 72.81 & 17.89 & 45.66 & 63.4\\
         \xrag & 45.50 & 68.69 & 72.54 & 67.89 & 15.71 & 44.60 & 52.49 \\
         \xragk & 60.77 & 73.13 & 53.22 & 70.24 & 14.61 & 44.63 & 52.77\\
         \approach & \underline{76.27} & 82.49 & \underline{96.32} & \underline{73.63} & \underline{23.59} & \underline{45.96} & \textbf{66.38}\\
         \hline
         \rowcolor{gray!20}
         \multicolumn{8}{c}{\centering{\textbf{\textit{Llama3.2-3B-Instruct}}}} \\
         \hline
         PT & 13.61 & 36.89 & 61.88 & 38.24 & 3.64 & 14.12 & 28.06\\
         IDPG & 35.28 & 87.14 & \underline{94.87} & 74.39 & 24.48 & 35.58 & 58.62\\
         \xrag & 33.31 & 71.40 & 68.22 & 71.31 & 20.84 & 36.06 & 50.19\\
         \xragk & 47.78 & 71.63 & 78.94 & 72.89 & 20.60 & 11.93 & 50.63\\
         \approach & \underline{67.62} & \underline{87.61} & 94.44 & \underline{80.73} & \underline{32.79} & \underline{36.43} & \textbf{66.60}\\
         \hline
         \rowcolor{gray!20}
         \multicolumn{8}{c}{\centering{\textbf{\textit{Mistral-7B-Instruct}}}} \\
         \hline
         PT & 54.97 & 33.97 & 0.0 & 33.88 & 5.10 & 22.34 & 25.04\\
         IDPG & 56.89 & 85.14 & 94.36 & 65.47 & 19.40 & 33.62 & 59.15\\
         \xrag & 71.43 & 69.46 & 74.35 & 67.69 & 18.93 & 33.32 & 55.86\\
         \xragk & 65.46 & 69.51 & 81.02 & 71.11 & 19.58 & 33.47 & 56.69\\
         \approach & \underline{76.26} & \underline{87.01} & \underline{94.42} & \underline{78.11} & \underline{34.54} & \underline{35.17} & \textbf{67.59}\\
         \bottomrule
    \end{tabular}
    \label{tab:baseline_soft_prompt}
\end{table*}
\textbf{Baselines.} Since our goal is to achieve the performance of adapter-style methods with the scalability of in-context learning (ICL), we include \textit{LoRA} \cite{hu2022lora} as a representative PEFT baseline that adapts the model via server-side updates for each user request. In addition, we evaluate two representative variants of ICL: \textit{fixed ICL}, where a shared set of exemplars is used across all queries, and \textit{adaptive ICL}, where relevant exemplars are retrieved per query from the training set. For brevity, we refer to \textit{adaptive ICL} simply as ICL in the remainder of the paper.

As shown in the Table \ref{tab:baseline_domain_adaptation}, \approach achieves the highest overall average improvement of $12.34$ over LoRA, with gains computed as the geometric mean across tasks and models. We next explain the performance gaps method by method. 

\textbf{Comparison with LoRA.} Most of the gains of \approach over LoRA arise on limited-data benchmarks (BACE, BBBP, and ClinTox), where LoRA often overfits, collapsing to predictions dominated by a single class. On medium-scale tasks (ChEBI and MMLU-Pro), \approach does not consistently outperform LoRA but remains competitive, significantly narrowing the gap compared to ICL: while ICL incurs an average drop of $-17.22$ relative to LoRA, \approach reduces this gap to $-2.99$ (geometric mean across models and tasks). From a systems perspective, serving multiple domain adaptation requests with LoRA requires maintaining separate adapter parameters per task (e.g., $\sim$1B parameters in total for Mistral-7B across benchmarks in Table \ref{tab:trainable_params_client_server}), whereas \approach requires no server-side parameter storage and only modest user-side overhead (less than 100M parameters on average). This makes \approach a more scalable alternative for multi-domain deployment.

 \textbf{Comparison with ICL.} The performance improvement of \approach over ICL stems from its flexible representational capacity: by varying the number of attention heads---each with its own set of weights---the model can specialise in capturing different types of dependencies across exemplars. Figure \ref{fig:flops_c_0} further shows that across $K=1\ldots10$, ICL incurs substantially higher inference cost (in FLOPs) than \approach, due to its longer context length and the quadratic scaling of attention with respect to $K$ (refer to Appendix \ref{sec:deploymentEfficiency} for a complete efficiency evaluation in terms of memory and inference time across tasks and models). Lastly, from Table \ref{tab:order_variance}, we observe that ICL also exhibits non-zero variance when the in-context exemplars are permuted. This sensitivity is expected because it concatenates exemplars as text whose order directly influences the model's input. In Appendix \ref{subsec:order_invariance_qwen3_0.6b} we found that this sensitivity is even worse for smaller-scale foundation models.

\textit{Findings. While MHA-ESP introduces a one-time, lightweight training step, it can be performed locally by each user, enabling scalable deployment across many tasks without server-side overhead as in LoRA, while improving performance where standard ICL is insufficient.}

\begin{wraptable}{r}{0.5\textwidth}
    \centering\footnotesize
    \caption{Trainable parameter distribution in a client–server deployment when adapting a foundation model to all benchmark tasks using LoRA versus \approach.}
    \begin{tabular}{lccc}
        \toprule
        \multicolumn{2}{c}{Model + Adaptation via} & Server-side & Client-side \\
        \midrule
        \multirow{2}{*}{Qwen3-4B} & LoRA & 792.72M & 0\\
        & \approach & 0 & 53.92M\\
        \hline
        \multirow{2}{*}{Llama3-3B} & LoRA & 583.56M & 0\\
        & \approach & 0 & 69.31M\\
        \hline
        \multirow{2}{*}{Mistral-7B} & LoRA & 1.0B & 0\\
        & \approach & 0 & 87.02M\\
        \bottomrule
    \end{tabular}
    \label{tab:trainable_params_client_server}
\end{wraptable}

\subsection{Comparison with Soft-Prompting Baselines}
\textbf{Baselines.} Since our method is soft-prompt-based, we benchmark against other soft-prompting approaches to identify the most effective design for domain adaptation. We first compare against standard soft-prompting baselines: (a) \textit{Prompt Tuning (\pt)}\cite{lester2021power}, which learns a fixed, question-independent prompt shared across the entire task, and (b) \textit{Instance-Dependent Prompt Generation (\idpg)}\cite{wu2022idpg}, which conditions the soft prompt on the input question. Next, we evaluate against baselines that derive soft prompts specifically from retrieved exemplars. We adapt the architecture originally proposed in \cite{cheng2024xrag} to create two distinct variants: (c) \textit{Unified-ESP (\xrag)}, a many-to-one mapping where the entire set of exemplars is represented with a single ``unified'' soft prompt embedding, and (d) \textit{Pointwise-ESP (\xragk)}, a one-to-one mapping where each individual exemplar is encoded into its own distinct ``point'' or soft prompt embedding.

Formally, given a set of $K$ exemplars, \xrag is defined as: $\mZ\in\mathcal{R}^{d\times 1}=\textrm{MLP}(E_{x_1\oplus y_1\oplus \ldots \oplus x_K\oplus y_K})$. In contrast, \xragk is defined as: $\mZ\in\mathcal{R}^{d\times K}=\textrm{MLP}(E_{x_1\oplus y_1})\oplus\ldots\oplus \textrm{MLP}(E_{x_K\oplus y_K})$. Both methods employ a single-layer MLP to project the soft prompts from the encoder's hidden space into the foundation model’s embedding space. Finally, \approach utilises a many-to-many mapping to capture complex inter-exemplar dependencies while generating the final soft prompt representation. For a detailed illustration and differences between \textsc{ESP}-variants, refer to Appendix \ref{appendix:soft-prompt-illustration}.

\begin{figure}[t]
    \begin{minipage}[b]{0.49\textwidth}
        \centering
        \includegraphics[width=\textwidth]{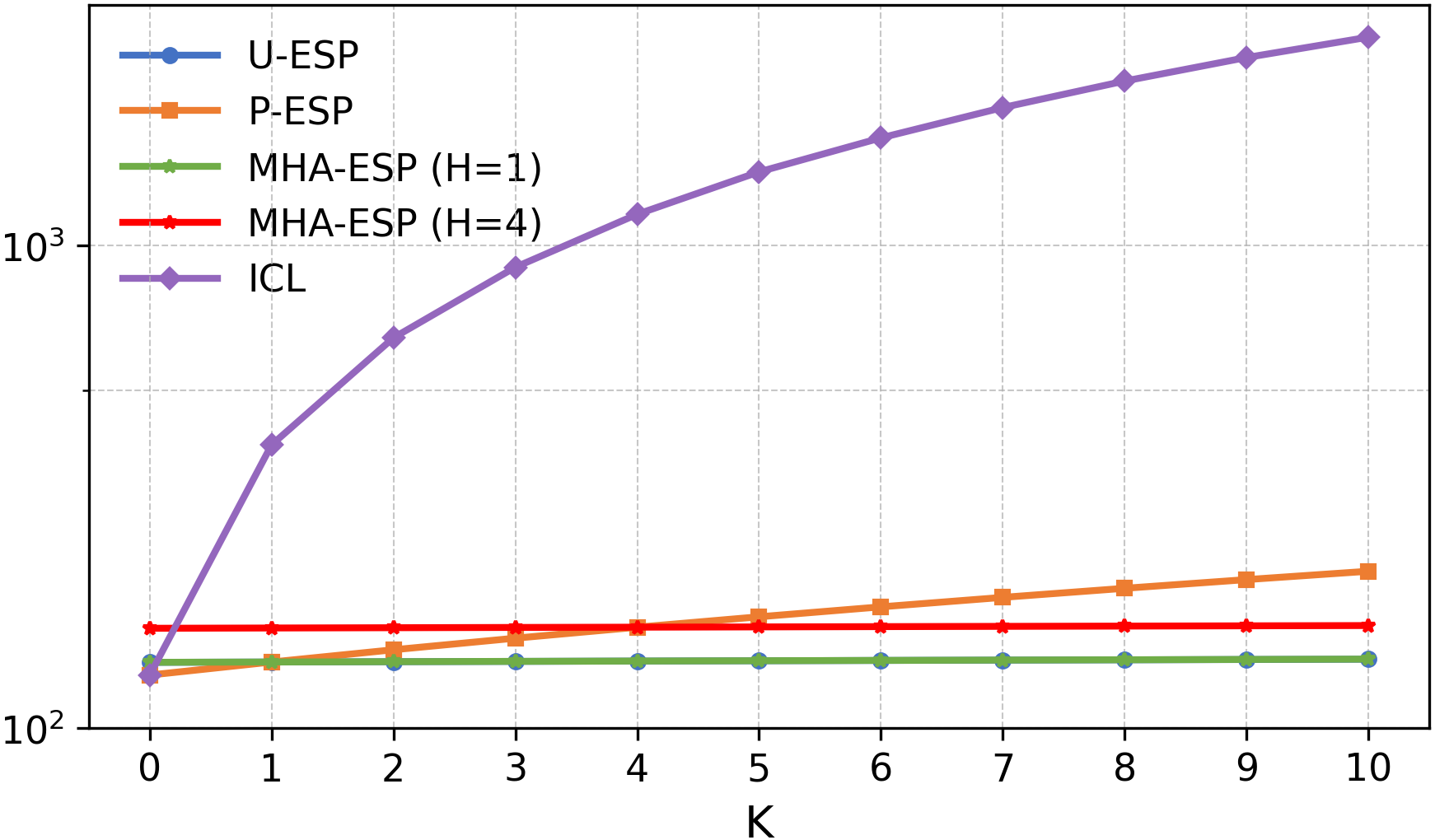}
        \captionof{figure}{Inference compute in FLOPs with $K$.}
        \label{fig:flops_c_0}
    \end{minipage}
    \hfill
    \begin{minipage}[b]{0.49\textwidth}
        \centering
        \captionof{table}{Order-(in)variance analysis for \textit{Qwen3-4B} with $K=5$: standard deviation in performance when exemplar order is randomized across 5 seeded shuffles. (Lower is better; zero means order-invariant). \textbf{Bold} shows the lowest standard deviation achieved.}
        \begin{tabular}{lcccc}
            \toprule
            \textbf{Tasks} & ICL & \xrag & \xragk & Ours\\
            \midrule
            BACE    & 3.18 & 2.80 & 1.12 & \textbf{0.0}\\
            BBBP    & 2.25 & 3.02 & 1.27 & \textbf{0.0}\\
            ClinTox & 8.52 & 0.12 & 4.91 & \textbf{0.0}\\
            \bottomrule
        \end{tabular}
        \label{tab:order_variance}
    \end{minipage}
\end{figure}

\textbf{Discussion.} From Table \ref{tab:baseline_soft_prompt}, we observe that \approach consistently outperforms both \pt and \idpg. This performance gap can be attributed to how soft prompts are constructed: \pt learns a fixed, question-independent prompt shared across the task, while \idpg conditions its soft prompt on the input question but does not leverage related training exemplars. In contrast, \approach conditions its soft prompts on retrieved exemplars most relevant to the query, enabling better utilisation of task-specific information.

\approach also outperforms \xrag and \xragk across nearly all configurations. The gap relative to \xrag likely arises from collapsing all exemplars into a single vector, leading to potentially lossy representations. While \xragk improves over \xrag by encoding each exemplar independently, all exemplars are processed with identical weights. In contrast, \approach assigns distinct weights to each head and computes representations via attention over the full set of exemplars, enabling head-specific representations.  These differences extend to efficiency and robustness: Figure~\ref{fig:flops_c_0} reports inference compute (in FLOPs) as a function of K, using \textit{ChemBERTa-2} as the encoder and \textit{Qwen3-4B} as the foundation model, where \approach with a single head matches the cost of \xrag and remains more efficient than \xragk for head counts <K. Both baselines are also sensitive to exemplar ordering (Table~\ref{tab:order_variance}): \xrag's joint encoding makes its representation depend on input order via encoding models's positional encodings, while \xragk's independent encoding determines the soft prompt's vector arrangement, producing downstream variance.

\textit{Findings. In domain adaptation tasks, MHA-ESP is the most effective soft-prompting-based method, delivering consistent performance gain across models and benchmarks.}

\begin{table*}[t]
    \centering
    \caption{Effective accuracy of \approach against ICL and \xrag under varying $K$. Results are averaged over 3 random seeds. \textbf{Bold} indicates the best performance. \underline{Underlined} indicates the value of $K$ that achieves the best effective accuracy for a given method.}
    \begin{tabular}{lccccccccc}
        \toprule
        \multirow{2}{*}{\textbf{Benchmarks}} & \multicolumn{3}{c}{ICL} & \multicolumn{3}{c}{\xrag} & \multicolumn{3}{c}{\approach}\\
        \cline{2-10}
         & \textit{K=1} & \textit{K=5} & \textit{K=10}  & \textit{K=1} & \textit{K=5} & \textit{K=10} & \textit{K=1} & \textit{K=5} & \textit{K=10}\\
         \hline
         \rowcolor{gray!20}
        \multicolumn{10}{c}{\centering{\textbf{\textit{Qwen3-4B}}}} \\
        \hline
        BACE  & 58.50 & \underline{75.87} & \underline{75.87} & \underline{61.35} & 45.35 & 34.03 & 59.89 & \underline{\textbf{76.27}} & 52.42\\
        BBBP & 67.50 & \underline{69.33} & 68.01 & \underline{73.51} & 68.69 & 70.43 & 72.12 & \underline{\textbf{82.49}} & 80.43\\
        ClinTox & 30.94 & 44.24 & \underline{54.48} & 47.64 & \underline{72.54} & 65.16 & 63.19 & \underline{\textbf{96.32}} & 89.44\\
        \hline
        \rowcolor{gray!20}
        \multicolumn{10}{c}{\centering{\textbf{\textit{Llama3.2-3B-Instruct}}}}\\
        \hline
        BACE & \underline{51.71} & 49.89 & 36.36 & \underline{39.77} & 33.31 & 16.29 & 52.37 & \underline{\textbf{67.62}} & 54.73 \\
        BBBP & 35.81 & 46.15 & \underline{52.08} & 69.80 & \underline{71.40} & 15.81 & 71.78 & \underline{\textbf{87.61}} & 86.80 \\
        ClinTox & 38.59 & 37.90 & \underline{47.27} & 67.35 & 68.22 & \underline{72.54} & 64.95 & \underline{\textbf{94.44}} & 89.33\\
        \bottomrule
    \end{tabular}
    \label{tab:varying_K}
\end{table*}

\section{Further Studies}
\label{sec:further_studies}

\textbf{Performance Sensitivity to Top-K Retrieval.} In this section, we examine how task performance varies with the number of retrieved exemplars ($K$), focusing on context saturation \cite{vladika2025influence}, the point at which adding more context introduces noise and causes performance to plateau or decline. As shown in Table \ref{tab:varying_K}, increasing $K$ generally improves effective accuracy for both ICL ($K=1 \to 5 \to 10$) and \approach ($K=1 \to 5$), reflecting the benefit of richer context. A key finding is that \approach at $K=5$ generally outperforms ICL and \xrag at their respective best-performing $K$s, indicating that it is more effective at extracting and representing information that is sufficient to answer the questions \cite{joren2024sufficient}. Moreover, it achieves this performance with fewer FLOPs than ICL (See Figure \ref{fig:flops_c_0}).
In table \ref{tab:varying_K}, we further observe that \approach reaches context saturation earlier than ICL: effective accuracy peaks around $K=5$, whereas ICL generally continues to improve up to $K=10$. Beyond these points, additional exemplars degrade effective accuracy, likely due to the inclusion of less relevant or noisy samples, consistent with findings of \cite{liu2023lost, zhao2023context, hsieh2024found}. A corresponding analysis showcasing that \approach is more robust to noisy retrieval than ICL can be found in Appendix \ref{sec:noise_robustnes}.

\textit{Findings. MHA-ESP enables more effective domain adaptation with fewer exemplars over ICL.}

\begin{figure*}[h]
    \centering
    \begin{subfigure}{0.45\textwidth}
        \centering
        \includegraphics[width=\linewidth]{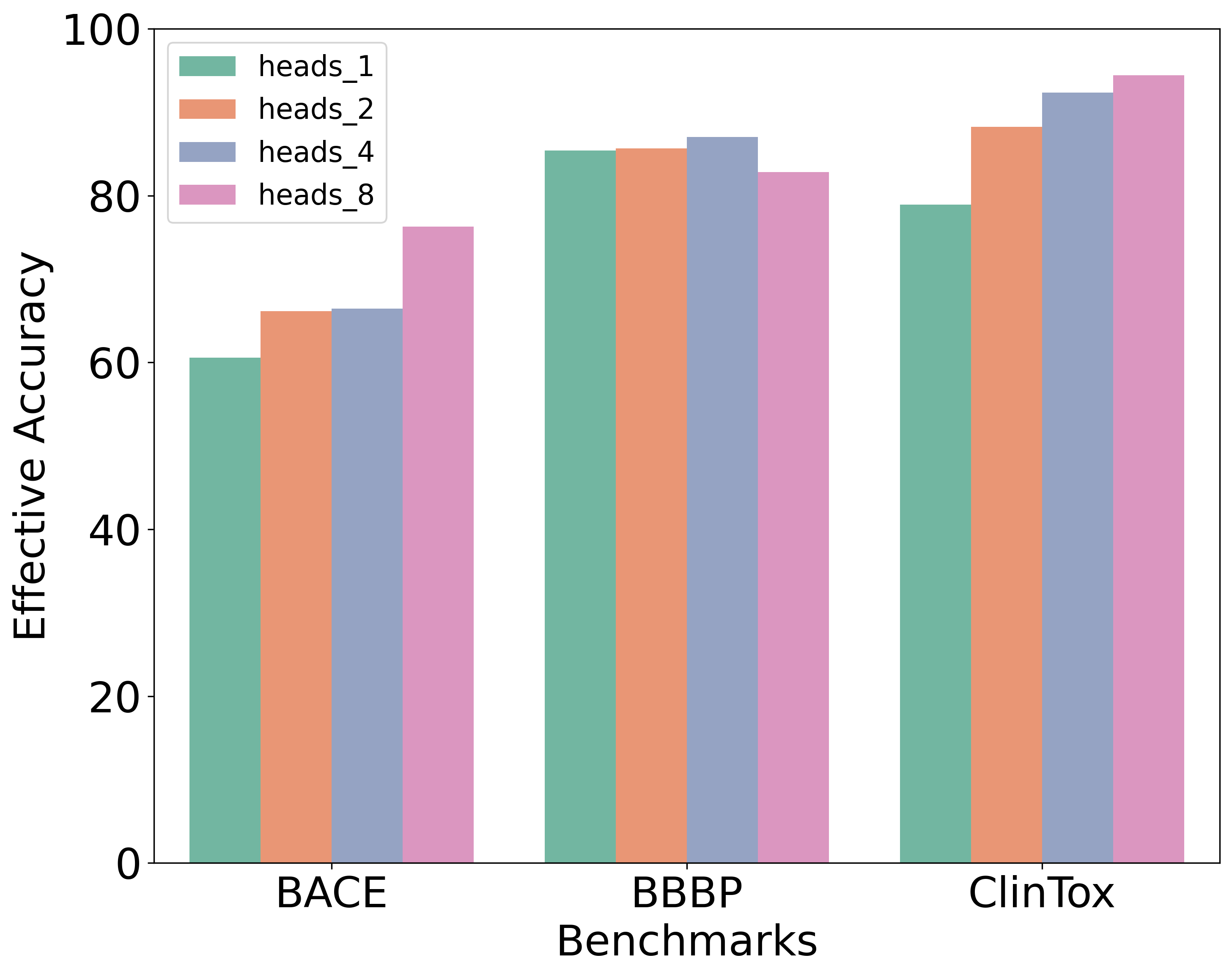}
    \end{subfigure}
    \hfill
    \begin{subfigure}{0.45\textwidth}
        \centering
        \includegraphics[width=\linewidth]{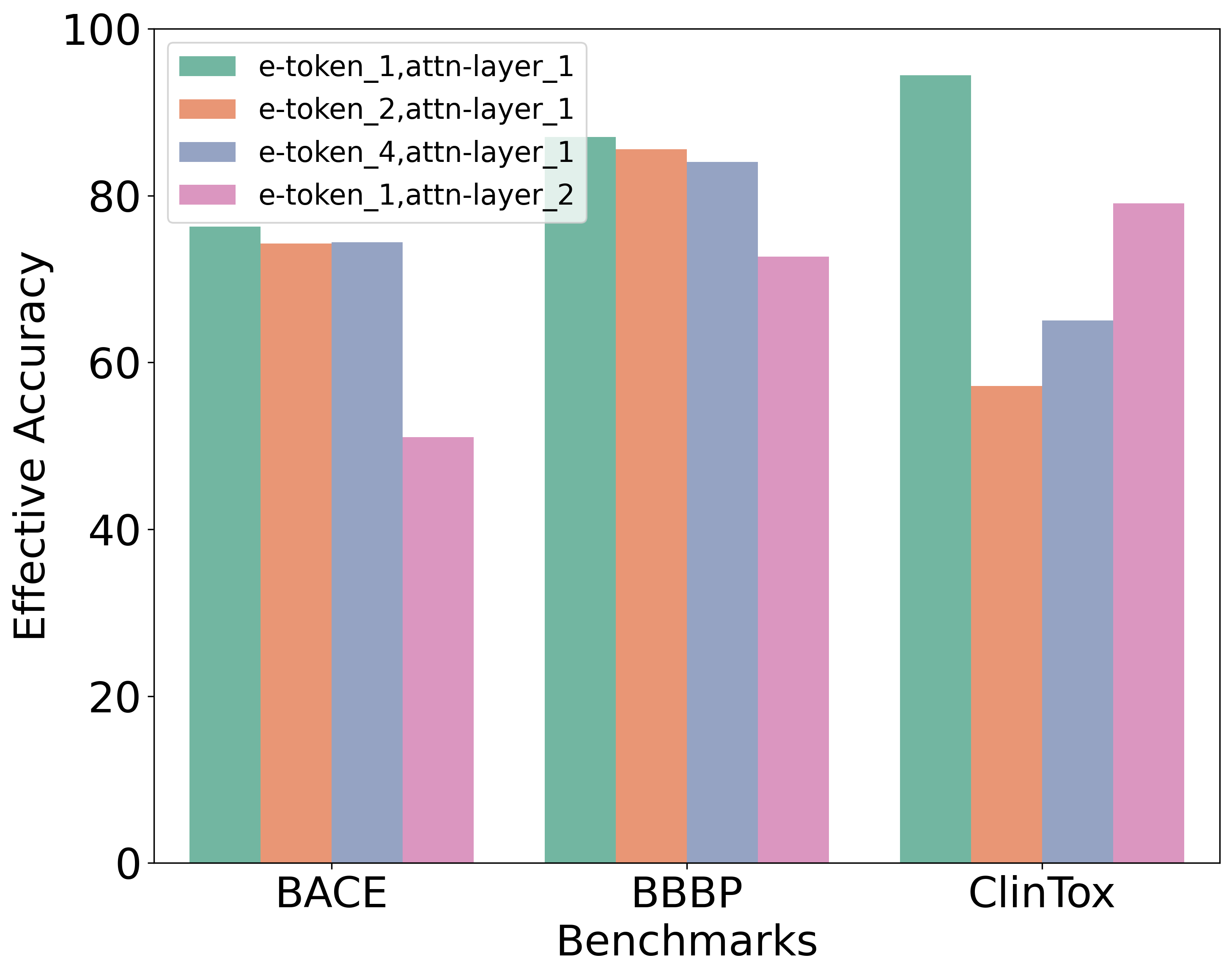}
    \end{subfigure}
    \caption{We fix the number of retrieved exemplars to five ($K=5$) and use Mistral-7B-Instruct as the foundation model. The reported performance evaluates the effect of varying: (left) the number of attention heads, (right) the number of tokens pooled from exemplars (e-tokens) and the number of self-attention layers (attn-layers) in the \approach framework.}
    \label{fig:design_space_explore}
\end{figure*}

\textbf{Design Space of \approach.} We study the design space of \approach along three axes: (1) the number of attention heads $H$, (2) the number of stacked multi-head attention (MHA) layers, and (3) the number of tokens used to represent each retrieved exemplar in the encoder-only language model. For (2), increasing the number of MHA layers corresponds to stacking multiple self-attention blocks in the computation of $g_{\varphi}(\sD_K \mid x)$, which applies self-attention iteratively over intermediate soft-prompt representations. For (3), we adopt the sliding-window chunking strategy \cite{hearst1997text} to obtain multiple tokens for representing a single exemplar.


We evaluate \approach on molecular-property-prediction tasks (BACE, BBBP, and ClinTox) using Mistral-7B-Instruct as the foundation model, with results summarised in Figure~\ref{fig:design_space_explore}. As shown there, effective accuracy improves with an increasing number of attention heads, indicating that with more heads, available contextual information is more effectively exploited as capacity is distributed across heads. Additional results are reported in Appendix~\ref{sec:more_plots_head_ablation}.

Now, we fix the number of attention heads and vary both the number of tokens per exemplar and the number of stacked MHA layers. We observe that representing each exemplar with a single token and using a single MHA layer achieves the best performance. While increasing attention depth or exemplar tokenization increases model expressiveness, we do not observe performance gain.

\textit{Findings. For the considered tasks, increasing the number of attention heads is more effective than increasing attention depth or per-exemplar tokenization.}

\begin{table*}[h]
    \centering
    \footnotesize
    \caption{Comparison of Off-the-shelf, ICL, and \approach methods across three models on in-domain and out-of-domain settings. Each cell reports mean $\pm$ standard deviation across three random in-domain/out-of-domain splits. \textbf{Bold} indicates the best performance per model.}
    \begin{tabular}{lcccccc}
    \toprule
        \multirow{2}{*}{\textbf{Method}} & \multicolumn{2}{c}{\textbf{Qwen3-4B}} & \multicolumn{2}{c}{\textbf{Llama3.2-3B-Instruct}} & \multicolumn{2}{c}{\textbf{Mistral-7B-Instruct}}\\
        \cline{2-7}
        & \textit{in-domain} & \textit{out-of-domain} & \textit{in-domain} & \textit{out-of-domain} & \textit{in-domain} & \textit{out-of-domain} \\
        \midrule
        Off-the-shelf  & 35.60$\pm$2.76  &  38.24$\pm$2.37 & 25.37$\pm$2.87 & 27.95$\pm$2.52 & 25.53$\pm$3.15 & 28.81$\pm$2.62\\
        ICL & 37.90$\pm$2.35 & 40.75$\pm$1.79 & 28.74$\pm$3.25 & 31.80$\pm$2.81 &  26.49$\pm$2.18 & 29.09$\pm$1.68\\
        \approach & \textbf{43.74$\pm$1.29} & \textbf{44.09$\pm$2.32} & \textbf{34.46$\pm$2.04} & \textbf{34.00$\pm$0.82} & \textbf{32.97$\pm$2.14} & \textbf{32.13$\pm$1.15}\\ 
        \bottomrule
    \end{tabular}
    \label{tab:mha_ood_generalisation}
\end{table*}

\textbf{Out-of-domain generalization.} In this section, we test out-of-domain generalisation of learned \approach modules. To serve the purpose, we consider the QA benchmark MMLU-Pro, which spans 14 subject domains. We partition these domains into two disjoint groups: 7 domains are used to train \approach (treated as \emph{in-domain} at test time), and the remaining 7 are held out entirely from training (\emph{out-of-domain}). To reduce sensitivity to the particular partition, we evaluate over three random splits. In Table \ref{tab:mha_ood_generalisation}, we observe that \approach consistently outperforms ICL across all foundation models in both in-domain and out-of-domain settings. This suggests that its learned multi-headed attention generalises effectively beyond training distributions, enabling robust adaptation to unseen domains. 

Refer to Appendix \ref{sec:c_within_k_exp} and \ref{sec:more_plots_head_ablation} for additional experiments.

\section{Conclusion}
\label{sec:conclusion}
In this paper, we proposed \approach, a soft-prompt-based method inspired by ICL, to use in-context exemplars to improve domain adaptation. \approach represents exemplars with soft prompts using an order-invariant architecture that avoids deployment overhead while maintaining low inference cost with stable performance. The number of attention heads serves as a tunable parameter, allowing efficient use of exemplar information and flexible control over soft-prompt generation. Compared to retrieval-augmented ICL and adapter-based fine-tuning methods, \approach offers a better balance between effectiveness, efficiency, and deployment-friendliness for domain adaptation. 

\textbf{Limitations and Future Works.} In this paper, \approach is proposed as a training-free approach that does not update the parameters of the foundation model; however, it still requires gradients to be back-propagated through the model. Future work could focus on alleviating this using derivative-free optimization (DFO) algorithms \cite{wierstra2014natural, rios2013derivative, qian2016derivative}, such as BBT gradient-free prompt tuning \cite{sun2022black, sun2022bbtv2}, for ESP-enabled Language-Model-as-a-Service (LMaaS). 


\bibliographystyle{plain}
\bibliography{neurips_paper}


\newpage
\appendix
\addcontentsline{toc}{section}{Appendix}
\startcontents 
\printcontents{}{1}{\section*{Table of Contents}} 

\newpage
\newpage
\section{Exemplar-Order Invariance}
\label{sec:order_invariance}
In this section, we argue that the soft prompt generated by \approach is invariant to the order of in-context exemplars/documents.
\newtheorem{claim}{Claim}
\begin{claim}[Exemplar-Order Invariance]\label{claim:order_invariance}
Let $\vq$ be a query vector for a given $x$ and for $k=1\ldots K$, let $\vk_k\in\mathcal{R}^d$ and $\vv_k\in\mathcal{R}^d$ be keys and value vectors derived from each retrieved example. Define the attention weights as $\alpha_k \coloneqq \frac{e^{s_k}}{\sum_{j=1}^K e^{s_j}}$ where score, $s_k \coloneqq \frac{\vq\cdot \vk_k}{\sqrt{d}}$. The attention output for each head is given by $\vz=\sum_{k=1}^K \alpha_k\, \vv_k$. Then $\vz$ is invariant under any permutation of the retrieved examples: permuting the indexing of the pairs $\{(\vk_k,\vv_k)\}_{k=1}^K$ does not change $\vz$.
\end{claim}

\begin{proof}
Let $\pi$ be a permutation of $1\ldots K$ and consider the permuted sequence $(\vk_{\pi(1)},\vv_{\pi(1)}),\ldots,(\vk_{\pi(K)},\vv_{\pi(K)})$. For the permuted sequence, the scores and weights are  
\begin{equation*}
    s'_{k}=\frac{\vq\cdot \vk_{\pi(k)}}{\sqrt{d}} = s_{\pi(k)} \qquad \alpha'_{k}=\frac{e^{s'_{k}}}{\sum_{j=1}^K e^{s'_{j}}}
= \frac{e^{s_{\pi(k)}}}{\sum_{j=1}^K e^{s_{\pi(j)}}}
\end{equation*}
Because $\{s_{\pi(j)}:j=1,\ldots,K\}$ is a reordering of $\{s_j:j=1,\ldots,K\}$, we have $\sum_{j=1}^K e^{s_{\pi(j)}}=\sum_{j=1}^K e^{s_j}$. Therefore, $\alpha'_{k}=\alpha_{\pi(k)}$.

Now, the head's resulting output under the permuted attention is $\vz' = \sum_{k=1}^K\alpha'_{k}\, \vv_{\pi(k)}=\sum_{k=1}^K \alpha_{\pi(k)}\, \vv_{\pi(k)}$. Without loss of generality, reindex the sum by setting $j=\pi(k)$. Because $\pi$ is a bijection, $k\mapsto j$ permutes the index set $\{1,\ldots,K\}$, resulting in $\vz'=\sum_{j=1}^K \alpha_j\, \vv_j=\vz$.

Thus $\vz$ is unchanged by the permutation $\pi$. The result holds for any permutation, so the \approach's per-head output is order invariant.
\end{proof}

\subsection{Additional Results}
\label{subsec:order_invariance_qwen3_0.6b}
\begin{table*}[h]
    \centering
    \caption{Order-(in)variance analysis for \textit{Qwen3-0.6B} with $K=5$: standard deviation in performance when exemplar order is randomized across 5 seeded shuffles. (Lower is better; zero means order-invariant). \textbf{Bold} shows the lowest standard deviation achieved.}
    \begin{tabular}{lcccc}
        \toprule
        \textbf{Benchmarks} &  ICL & \xrag & \xragk & Ours\\
        \midrule
        BACE & 5.36 & 1.09 & 1.08 & \textbf{0.0}\\
        BBBP & 2.83 & 1.23 & 2.66 & \textbf{0.0}\\
        ClinTox & 29.31 & 6.34 & 3.90 & \textbf{0.0}\\
        \bottomrule
    \end{tabular}
    \label{tab:order_variance_qwen_small}
\end{table*}
We observe that, compared to Table \ref{tab:order_variance}, Qwen3-0.6B with ICL in Table \ref{tab:order_variance_qwen_small} is much more sensitive to the order of exemplars, leading to greater performance degradation. In comparison, \approach, by design, is order-invariant and independent of model scale.

\newpage
\section{Illustration: Exemplar Soft Prompt}
\label{appendix:soft-prompt-illustration}
\begin{figure}[h]
    \centering
    \includegraphics[width=\linewidth]{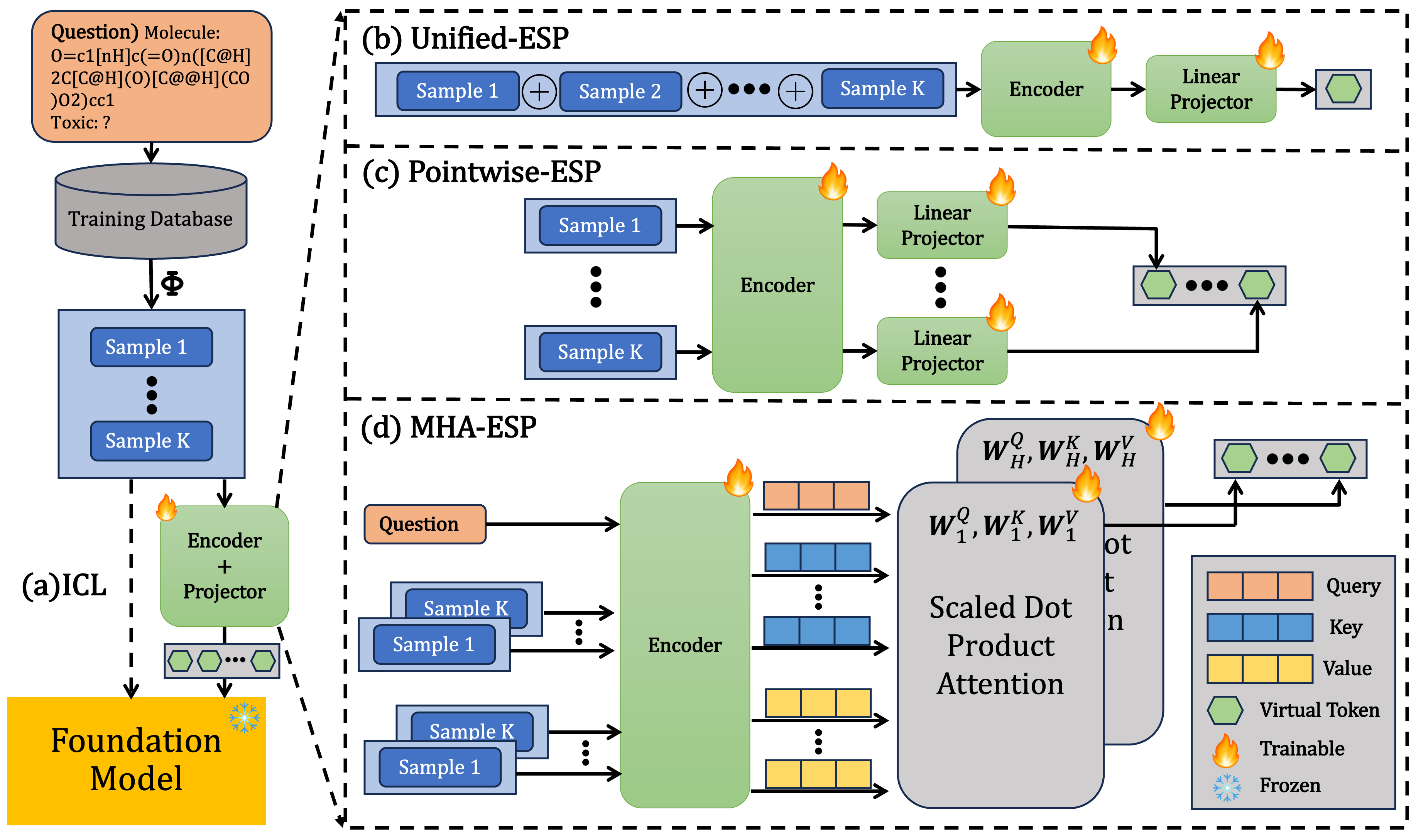}
    \caption{Comparison of exemplar-driven domain-adaptation methods: (a) ICL uses retrieved exemplars directly as context, while (b) \xrag, (c) \xragk, and (d) \approach derive soft-prompt representations from exemplars. The figure illustrates how each method forms its representation (single vector in \xrag, $K$ vectors in \xragk, multi-head representations in \approach) and highlights their trainable components. Among the four methods, only \approach---due to its use of scaled dot-product attention---is invariant to the order of exemplars.
    }
    \label{fig:mha_rag_illustration}
\end{figure}

While \xragk and \approach share the idea of representing retrieved exemplars as multiple embeddings, \approach introduces three fundamental differences:
\begin{itemize}
    \item \textbf{Joint representation vs. independent encoding.} \xragk encodes each exemplar independently via a shared MLP, and the resulting embeddings are concatenated. In contrast, \approach produces each soft prompt slot via an attention head that jointly attends over all retrieved exemplars, enabling each head to extract query-relevant information (analogous to multi-head attention in Transformers), rather than treating exemplars in isolation. This distinction is key to the improved domain adaptation performance observed in Table \ref{tab:baseline_soft_prompt}.
    \item \textbf{Flexible prompt capacity.} In \xragk, the number of retrieved exemplars (K) directly determines the number of soft prompt embeddings. \approach decouples these factors: the number of attention heads controls the soft prompt size, making it a tunable hyperparameter independent of K and offering greater flexibility in balancing performance and efficiency.
    \item \textbf{Order invariance vs. order sensitivity.} \xragk preserves exemplar order when constructing the soft prompt, leading to sensitivity to exemplar ordering (as evidenced in Table \ref{tab:order_variance}). In contrast, \approach aggregates information through attention, making it inherently order-invariant by design and thus more robust.
\end{itemize}


\newpage
\section{Benchmark: Inequality Math}
\label{sec:math_benchmark}
\begin{table*}[h]
    \centering
    \caption{Baseline comparison on \textit{IneqMath} requiring step-by-step reasoning (with $K=2$). We report the Final Accuracy, measured as the exact match between derived numerical values and ground truth. \underline{Underline} indicates the best performance.}
    \begin{tabular}{lcccc}
        \toprule
        Accuracy & Off-the-Shelf & w ICL & w LoRA & w \approach \\
        \hline
        \rowcolor{gray!20}
        \multicolumn{5}{c}{\centering{\textbf{\textit{Llama3.2-1B-Instruct}}}} \\
        \hline
        Bound & 14.0 & 4.0 & \underline{12.0} & \underline{12.0}  \\
        Relation & 12.0 & 16.0 & \underline{26.0} & 20.0  \\
        Final & 13.0 & 10.0 & \underline{19.0} & 16.0  \\
        \hline
        \rowcolor{gray!20}
        \multicolumn{5}{c}{\centering{\textbf{\textit{Llama3.2-3B-Instruct}}}} \\
        \hline
        Bound & 10.00 & 10.0 & 12.0 & \underline{14.0}\\
        Relation & 22.00 & 24.0 & \underline{26.0} & 22.0\\
        Final & 16.00 & 17.0 & \underline{19.0} & 18.0\\
        \bottomrule
    \end{tabular}
    \label{tab:pubLong_IneqMath}
\end{table*}

In this section, we investigate whether \approach can improve performance on mathematical-reasoning tasks. Specifically, we evaluate on \textit{Inequality Math} [train/dev split: 1252/100], an Olympiad-level benchmark that requires proving bounds-preserving inequalities and establishing relations between algebraic expressions \cite{sheng2025solving}. For retrieval, we compute dense representations of question-exemplar pairs using \textit{Qwen3-Embedding-8B} \cite{zhang2025qwen3}, and rank candidates by cosine similarity.

We primarily report results for \textit{Llama3.2} family, because models in the Qwen3 family already achieve high off-the-shelf accuracy (e.g., \textit{Qwen3-4B} at $65\%$), and adding retrieved exemplars led to performance drops. This observation is consistent with \cite{sheng2025solving}, which found that only certain model families benefit from in-domain exemplars. A plausible explanation is that high-performing models may have already been exposed to mathematical-reasoning tasks during pre-training, reducing the utility of retrieval.

From Table \ref{tab:pubLong_IneqMath}, we observe that \approach yields improvements over both off-the-shelf and ICL. A possible explanation is that the \textit{Inequality Math} dataset contains multiple problems that rely on the same theorems or follow similar reasoning steps. If retrieval fails to surface such structurally related exemplars, the model cannot fully benefit from exemplar conditioning. Future work could therefore focus on improving retrieval quality—e.g., by incorporating reasoning-aware similarity metrics—which may allow \approach to better exploit shared problem structure and yield stronger overall gains.

\newpage
\section{Performance benchmarking on small-scale foundation model}
\label{sec:qwen3_0.6b}
In this section, we compare the performance of Qwen3-0.6B with the proposed \approach against baselines on our set of benchmark tasks.

\begin{table*}[h]
    \centering
     \caption{Baseline comparison with domain adaptation methods, including soft-prompt based. Hyperparameters are tuned via sweeps: Prompt Tuning (virtual tokens $\in \{1, 5, 10\}$), LoRA (rank $\in \{16, 32, 64\}$), and \approach ($H\in \{1, 2, 4, 8\}$). \textbf{Bold} indicates the best effective accuracy.}
    \begin{tabular}{lcccccc}
        \toprule
        & \multicolumn{3}{c}{\textbf{Property Prediction}} & \multicolumn{2}{c}{\textbf{CHEBI}} &\\
        \cline{2-6}
        \textbf{Methods} & \textit{BACE} & \textit{BBBP} & \textit{ClinTox} & \textit{Design} & \textit{Captioning}  & \textbf{Avg.}\\
        \midrule
         Off-the-Shelf & 36.67 & 22.21 & 0.0 & 32.20 & 0.34 & 18.28\\
         w ICL & \textbf{76.01} & 66.62 & 53.30 & 61.47 & 6.50 & 52.78\\
         w PT & 0.0 & 0.0 & 0.0 & 0.44 & 0.0 & 0.09\\
         w IDPG & 53.75 & 86.52 & 94.87 & 60.74 & 2.11 & 59.60\\
         w LoRA & 0.0 & 0.0 & 0.0 & \textbf{66.50} & \textbf{20.38} & 17.38\\
         w \xrag & 40.95 & 68.81 & 31.39 & 29.34 & 4.68 & 35.03 \\
         w \xragk & 52.27 & 70.92 & 35.70 & 35.84 & 4.57 & 39.86 \\
         w \approach & 75.08 & \textbf{87.82} & \textbf{97.12} & 58.39 & 15.22  & \textbf{66.73} \\
         \bottomrule
    \end{tabular}
    \label{tab:baseline_qwen3_0.6b}
\end{table*}

\section{Robustness to noisy and adversarial exemplars}
\label{sec:noise_robustnes}
\begin{table*}[h]
    \centering
    \caption{Robustness comparison under random exemplar replacement. We report performance averaged across 3 seeded runs with Top-$K$ retrieval, Random-$K$ exemplar replacement, and the corresponding performance drop ($\Delta$) for ICL and the proposed \approach across three benchmark tasks and two foundation models.}
    \begin{tabular}{l l ccc ccc}
        \toprule
        \multirow{2}{*}{\textbf{Base Model}} &
        \multirow{2}{*}{\textbf{Task}} &
        \multicolumn{3}{c}{\textbf{ICL}} &
        \multicolumn{3}{c}{\textbf{\approach}} \\
        \cmidrule(lr){3-5} \cmidrule(lr){6-8}
        & &
        Top-$K$ & Random-$K$ & $\Delta$ &
        Top-$K$ & Random-$K$ & $\Delta$ \\
        \midrule
        
        \multirow{3}{*}{Qwen3-4B}
        & BACE & 75.87 & 54.37 & -21.50 & 76.27 & 58.31 & -17.96 \\
        & BBBP & 69.33 & 59.36 & -9.97 & 82.49 & 68.26 & -14.23 \\
        & ClinTox & 44.24 & 0.0 & -44.24 & 96.32  & 83.22 & -13.10 \\
        
        \midrule
        
        \multirow{3}{*}{Llama3.2-3B}
        & BACE & 49.89 & 52.45 & +2.56 & 67.62 & 52.49 & -15.13 \\
        & BBBP & 46.15 & 50.74 & -4.59 & 87.61 & 64.55 & -23.06 \\
        & ClinTox & 37.90 & 39.40 & -1.50 & 94.44 & 88.97 & -5.47 \\
        
        \bottomrule
    \end{tabular}
    \label{tab:robustness_simple}
\end{table*}

To assess robustness to noisy exemplars, we conducted an experiment where exemplars were randomly sampled from the database instead of selecting the top-5 most relevant ones. We observed that both ICL and \approach experience a performance drop in this setting, particularly with Qwen3-4B; however, \approach still outperforms ICL after random selection. This indicates that (1) exemplar relevance plays a critical role in effective question answering, and (2) \approach is better at representing and leveraging even randomly sampled (i.e., noisy) exemplars than ICL.

Interestingly, for Llama3.2-3B-Instruct, ICL shows a slight improvement when using random exemplars. Overall, \approach continues to outperform ICL under both random and relevant exemplar settings, further supporting its superior representational capability for skill acquisition in new domains and tasks.

\newpage
\section{Additional Details: Experimental Setup}
\label{sec:hyperParamSearch}

\subsection{Optimising number of heads in \approach}

\begin{table*}[h]
    \centering
    \footnotesize
    \caption{Choice of hyper-parameter $m$ i.e. number of heads used by \approach for experiments conducted in Table \ref{tab:baseline_domain_adaptation}, \ref{tab:baseline_soft_prompt}} 
    \begin{tabular}{lcccccc}
        \toprule
           Base Models  & BACE & BBBP & ClinTox & Design & Captioning  & MMLU-Pro\\
           \midrule
           Qwen3-0.6B & 4 & 8 & 8 & 8 & 8 & 2\\
           Qwen3-4B & 2 & 8 & 4 & 8 & 8 & 4\\
           Llama3.2-3B-Instruct & 2 & 8 & 8 & 8 & 8 & 8\\
           Mistral-7B-Instruct & 8 & 4 & 8 & 8 & 8 & 8\\
        \bottomrule
    \end{tabular}
    \label{tab:choice_of_m}
\end{table*}

\subsection{Optimising ICL}
In this section, we consider several key design dimensions to optimise ICL performance in our benchmark tasks.

\begin{table*}[h]
    \centering
    \caption{Performance variance of Qwen3-4B with ICL $(K=5)$ with respect to prompt engineering and exemplar order}
    \begin{tabular}{lccc}
        \toprule
        Design Dimension & BACE & BBBP & ClinTox \\
        \midrule
        Prompt Engineering & 77.0$\pm$1.0 & 66.0$\pm$1.0 & 44.0$\pm$0.0\\
        Exemplar Order & 70.0$\pm$3.2 & 70.0$\pm$2.3 & 41.1$\pm$8.5\\
        \bottomrule
    \end{tabular}
    \label{tab:optimise_icl}
\end{table*}

\textbf{Prompt engineering.} We evaluate Qwen3-4B (K=5) with ICL using five diverse prompts across three tasks, varying in reasoning style and domain specificity (e.g., biology-grounded explanations vs. rule-based cheminformatics reasoning), while targeting the same objective. We observe that prompt choice has a negligible impact on performance, suggesting that our results are not sensitive to prompt design.

\textbf{Exemplar ordering.} Table \ref{tab:order_variance} reports exemplar-order variance for ICL, with the corresponding averages for Qwen3-4B (K=5) reported here. While favourable orderings can yield modest gains, the improvements remain significantly smaller than those achieved by \approach. Moreover, we emphasize that sensitivity to exemplar ordering is itself a limitation of standard ICL, which \approach explicitly addresses through order-invariant aggregation.

\textbf{Text truncation/compression strategies.} Prior hard-prompt compression methods (e.g., EXIT, COCOM) operate by pruning irrelevant content from retrieved text. In our setting, the retrieved context consists of relevant input–output exemplars that do not contain redundant text to be pruned. An analogous filtering step is already performed by the retrieval function, which selects the top-K most relevant exemplars. We study the effect of context size on ICL via ablation over K in Table \ref{tab:varying_K}. 

\textbf{Retrieval quality.} To assess the impact of retrieval, refer to the experiment in Appendix \ref{sec:noise_robustnes} where exemplars are randomly sampled. Both ICL and \approach exhibit performance drops (up to 44.24), confirming that retrieval quality is critical and that our retrieval setup is meaningful and non-trivial.

\subsection{Training Details}
Training is done for 10 epochs on the limited-data molecular benchmarks (\textit{BACE}, \textit{BBBP}, \textit{ClinTox}), 4 epochs (\textit{ChEBI-Design} and \textit{ChEBI-Captioning}) and 2 epochs (\textit{MMLU-Pro}) on the medium-scale benchmarks. Learning rates are selected from $\{1e-5, 3e-5, 5e-4\}$. The execution platform used 8× V100 GPUs (32GB VRAM each). The embedding models are fine-tuned with LoRA (rank 64).

\subsection{OOD Generalisation}
For all baselines, we evaluate per-domain accuracy and average across the 7 in-domain and 7 out-of-domain categories separately. \approach is trained with $K{=}5$ retrieved exemplars, $m'{=}8$, learning rate $3\mathrm{e}{-}5$, for 2 epochs. We report results on three backbones: Qwen3-4B, Llama3.2-3B, and Mistral-7B. The in-domain training areas across the three splits considered in Table \ref{tab:mha_ood_generalisation} are:
\begin{itemize}
    \item \textbf{Split 1:} law, other, economics, health, psychology, business, history
    \item \textbf{Split 2:} math, physics, chemistry, engineering, biology, computer science, philosophy
    \item \textbf{Split 3:} math, chemistry, law, engineering, other, economics, history
\end{itemize}

\newpage
\section{Deployment Efficiency: Inference Runtime \& Memory Footprint}
\label{sec:deploymentEfficiency}

\begin{table*}[th]
    \centering
    \footnotesize
    \caption{Comparison of inference efficiency—measured by average inference time (ms) and peak memory consumption (MB)—when deploying \textit{Llama3.2-3B-Instruct}, \textit{Qwen3-4B} and \textit{Mistral-7b-Instruct} with different domain adaptation baselines on benchmark tasks. For \approach, we report the metrics corresponding to the best $m\in\{1,2,4,8\}$}
    \begin{tabular}{clccccc}
        \toprule
        Task & Metrics & Off-the-Shelf & w ICL & w \xrag & w \xragk & \approach\\
        \midrule
        \rowcolor{gray!20}
        \multicolumn{7}{c}{\centering{\textbf{\textit{Qwen3-4B}}}} \\
        \hline
        \multirow{2}{*}{BACE} & Peak Memory (MB) & 7822.7 & 8041.2 & 7832.4 & 7834.2 & 7880.4\\
         & Avg. Inference Time (ms) & 0.0551 & 0.0939 & 0.0511 & 0.0522 & 0.0536\\
        \hline
        \multirow{2}{*}{BBBP} & Peak Memory (MB) & 7807.9 & 8109.8 & 7817.1 & 7818.4 & 7865.9\\
         & Avg. Inference Time (ms) & 0.0485 & 0.0872 & 0.0490 & 0.0524 & 0.0497\\
        \hline
        \multirow{2}{*}{ClinTox} & Peak Memory (MB) & 7883.6 & 8492.4 & 7893.4 & 7895.0 & 7941.5\\
         & Avg. Inference Time (ms) & 0.0535 & 0.1046 & 0.0543 & 0.0584 & 0.0580\\
        \hline
        \multirow{2}{*}{Chebi-captioning} & Peak Memory (MB) & 7919.5 & 12223.0 & 10887.4 & 10413.3 & 10528.4 \\
         & Avg. Inference Time (ms) & 0.8571 & 0.8800 & 0.8538 & 0.9610 & 0.8581 \\
        \hline
        \multirow{2}{*}{Chebi-design} & Peak Memory (MB) & 7874.3 & 11728.3 & 10894.0 & 10415.6 & 10530.8 \\
         & Avg. Inference Time (ms) & 1.8095 & 1.4451 & 1.3964 & 1.6220 & 1.1898 \\

        \midrule
        \rowcolor{gray!20}
        \multicolumn{7}{c}{\centering{\textbf{\textit{Llama3.2-3B-Instruct}}}} \\
        \hline
        \multirow{2}{*}{BACE} & Peak Memory (MB) & 6206.0 & 6365.8 & 6215.2 & 6216.8 & 6268.67\\
         & Avg. Inference Time (ms) & 0.0411 & 0.0708 & 0.0426 & 0.0433 & 0.0429\\
        \hline
        \multirow{2}{*}{BBBP} & Peak Memory (MB) & 6196.8 & 6398.2 & 6206.2 & 6207.0 & 6259.2\\
         & Avg. Inference Time (ms) & 0.0386 & 0.0684 & 0.0402 & 0.0430 & 0.0443\\
        \hline
        \multirow{2}{*}{ClinTox} & Peak Memory (MB) & 6247.9 & 6692.8 & 6257.5 & 6259.1 & 6312.7\\
         & Avg. Inference Time (ms) & 0.0436 & 0.0744 & 0.0442 & 0.0459 & 0.0443\\
        \hline
        \multirow{2}{*}{Chebi-captioning} & Peak Memory (MB) & 6278.0 & 9439.4 & 9290.4 & 8815.5 & 8953.7 \\
         & Avg. Inference Time (ms) & 0.3683 & 0.5491 & 0.5563 & 0.5597 & 0.6348 \\
        \hline
        \multirow{2}{*}{Chebi-design} & Peak Memory (MB) & 6259.2 & 9084.5 & 9296.9 & 8818.0 & 8956.1 \\
         & Avg. Inference Time (ms) & 0.9929 & 1.0450 & 1.0830 & 0.8731 & 0.7681 \\
        
        \midrule
        \rowcolor{gray!20}
        \multicolumn{7}{c}{\centering{\textbf{\textit{Mistral-7B-Instruct}}}} \\
        \hline
        \multirow{2}{*}{BACE} & Peak Memory (MB) & 13925.1 & 14128.3 & 13935.9 & 13938.0 & 14007.0\\
         & Avg. Inference Time (ms) & 0.2455 & 0.1523 & 0.0746 & 0.0805  & 0.0801\\
        \hline
        \multirow{2}{*}{BBBP} & Peak Memory (MB) & 13910.3 & 14226.2 & 13921.2 & 13922.0 & 13990.8\\
         & Avg. Inference Time (ms) & 0.2315 & 0.1535 & 0.0678 & 0.0719 & 0.0694\\
        \hline
        \multirow{2}{*}{ClinTox} & Peak Memory (MB) & 13999.9 & 14624.3 & 14010.9 & 14012.7 & 14084.0\\
         & Avg. Inference Time (ms) & 0.2526 & 0.1720 & 0.0793 & 0.0823 & 0.0883\\
        \hline
        \multirow{2}{*}{Chebi-captioning} & Peak Memory (MB) & 14028.6 & 19011.2 & 16988.6 & 16513.2 & 16697.4 \\
         & Avg. Inference Time (ms) & 0.7827 & 0.8702 & 0.6394 & 0.6409 & 0.6593 \\
        \hline
        \multirow{2}{*}{Chebi-design} & Peak Memory (MB) & 13964.5 & 18492.8 & 16995.1 & 16515.4 & 16699.6 \\
         & Avg. Inference Time (ms) & 1.2191 & 1.3008 & 1.0964 & 0.9834 & 0.9061 \\
        \bottomrule
    \end{tabular}
    
\label{tab:deployment_efficiency}
\end{table*}

In this section, we compare the inference efficiency of a base model adapted using different training methods, focusing on average inference latency (ms) and peak GPU memory consumption (GB) measured on identical hardware (Table \ref{tab:deployment_efficiency}). All measurements are obtained with $K=5$ for all retrieval-based baselines.

Across this setting, we observe that representing retrieved exemplars via \approach consistently yields lower inference latency and reduced GPU memory usage compared to standard ICL, a trend that holds across varying numbers of attention heads. Furthermore, \approach exhibits a comparable inference computational footprint to \xrag and \xrag, while achieving superior predictive performance (see Table \ref{tab:baseline_soft_prompt}).

\newpage
\section{Experiment: Does Adding Exemplars Back Help?}
\label{sec:c_within_k_exp}
\begin{table*}[th]
    \centering
    \caption{Effect of re-inserting top-$c$ exemplars in textual form (up to $c=5$, given a fixed inference budget) into the in-context prompt while still using top-$K$ for soft-prompt computation in \approach. The $K=5$, $c=0$ column corresponds to the \approach column of Table \ref{tab:baseline_domain_adaptation}.
    \textbf{Bold} indicates the best effective accuracy.
    \underline{Underlined} indicates the value of $c$ that achieves the best effective accuracy for a given $K$.}
    \begin{tabular}{lccccccc}
        \toprule
        \multirow{2}{*}{\textbf{Benchmarks}} & \multicolumn{3}{c}{\textit{K=5}} & \multicolumn{3}{c}{\textit{K=10}}\\
        \cline{2-7}
         & \textit{c=0} & \textit{c=1} & \textit{c=5} & \textit{c=0} & \textit{c=1} & \textit{c=5}\\
        \hline
        \rowcolor{gray!20}
        \multicolumn{7}{c}{\centering{\textbf{\textit{Qwen3-4B}}}}\\
        \hline
        BACE & 76.27 & 75.27 & \underline{79.03} & 52.42 & 76.27 & \underline{\textbf{79.19}}\\
        BBBP & 82.49 & 68.56 & \underline{\textbf{88.14}} & 80.43 & 70.01 & \underline{87.5}\\
        ClinTox & \underline{96.32} & 76.63 & 94.55 & 89.44 & 70.72 & \underline{\textbf{99.04}}\\
        \hline
        \rowcolor{gray!20}
        \multicolumn{7}{c}{\centering{\textbf{\textit{Llama3.2-3B-Instruct}}}}\\
        \hline
        BACE & 67.62 & \underline{\textbf{79.96}} & 78.82 & 54.73 & 74.23 & \underline{77.75}\\
        BBBP & 87.61 & 88.99 & \underline{\textbf{89.74}} & 86.80 & 85.44 & \underline{86.94}\\
        ClinTox & 94.44 & 88.97 & \underline{\textbf{98.40}} & 89.33 & \underline{94.49} & 88.97\\
        \bottomrule
    \end{tabular}
    \label{tab:topC_within_topK}
\end{table*}
To investigate whether performance can be further improved, we augment \approach’s soft prompts with a small number of exemplars directly included in the model’s context at inference. Specifically, we add the top-$c$ retrieved exemplars alongside the soft prompts and study the effect across different $K$ values.

\begin{figure}[h]
    \centering
    \includegraphics[width=0.8\linewidth]{ 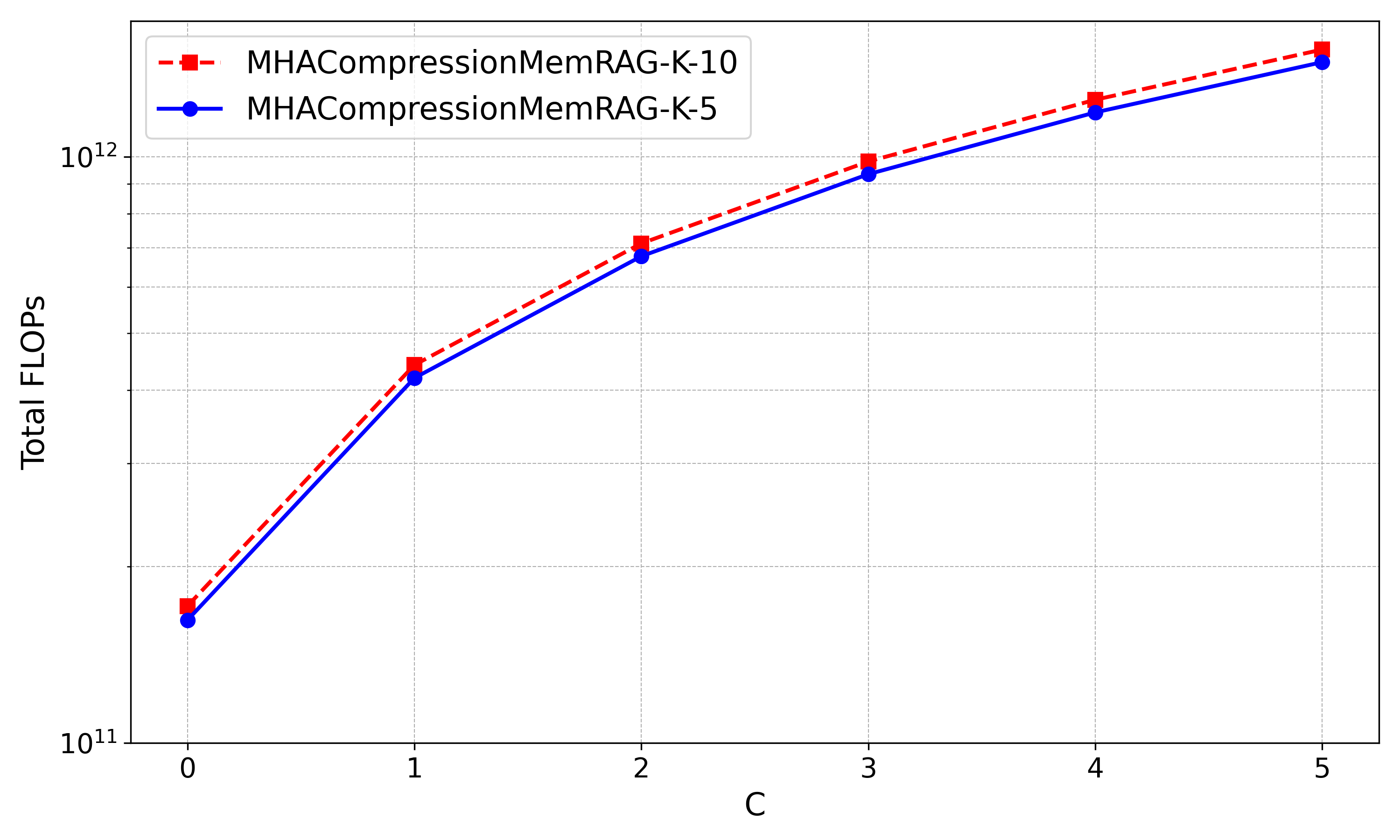}
    \caption{
    Total FLOPs for inference with encoder ChemBERTa-10M-MTR and foundation model Qwen3-4B using an increasing number of exemplars $c$ in the context, given
    $K=5$ and $K=10$ as input to create a soft prompt.}
    \label{fig:topc_within_k_ablation}
\end{figure}

As shown in Table \ref{tab:topC_within_topK}, effective accuracy improves consistently with increasing $c$, with the best results achieved at $c=5$ for both $K=5$ and $K=10$. The FLOPs analysis in Figure \ref{fig:topc_within_k_ablation} reveals that increasing $K$ from 5 to 10 incurs only a minor cost, while increasing $c$ from 0 to 5 leads to a logarithmic rise in FLOPs. These results suggest that $c$ can serve as an additional tunable knob for balancing accuracy gains against inference cost, allowing practitioners to adapt \approach to different computational budgets.

\newpage
\section{Experiment: Additional ablation with number of heads}
\label{sec:more_plots_head_ablation}
\begin{figure*}[h]
    \centering\footnotesize
    \begin{subfigure}{0.49\textwidth}
        \centering
        \includegraphics[width=\linewidth]{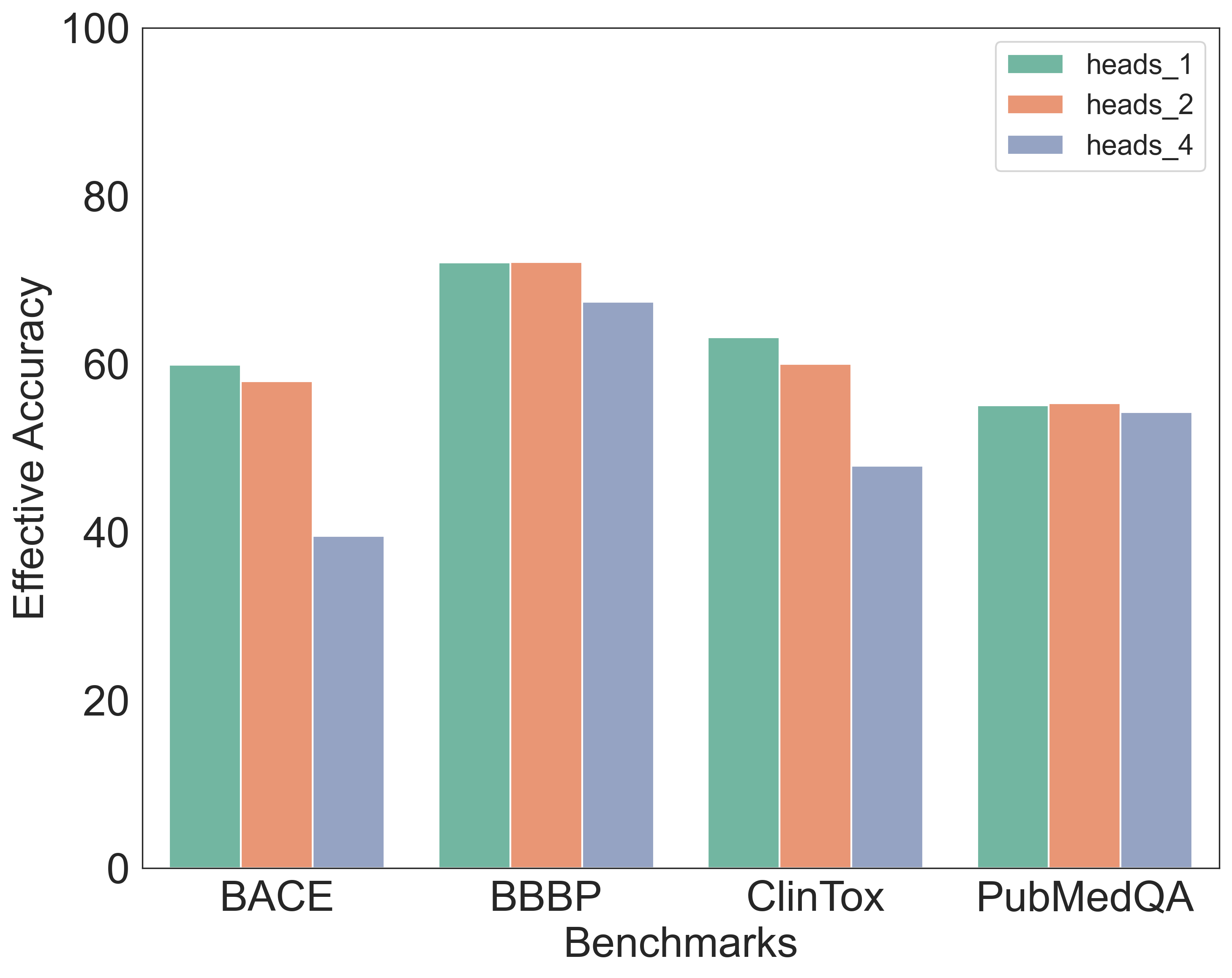}
        \caption{$K=1$, \textit{Qwen3-4B}}
        \label{fig:qwen4b_k_1}
    \end{subfigure}
    \hfill
    \begin{subfigure}{0.49\textwidth}
        \centering
        \includegraphics[width=\linewidth]{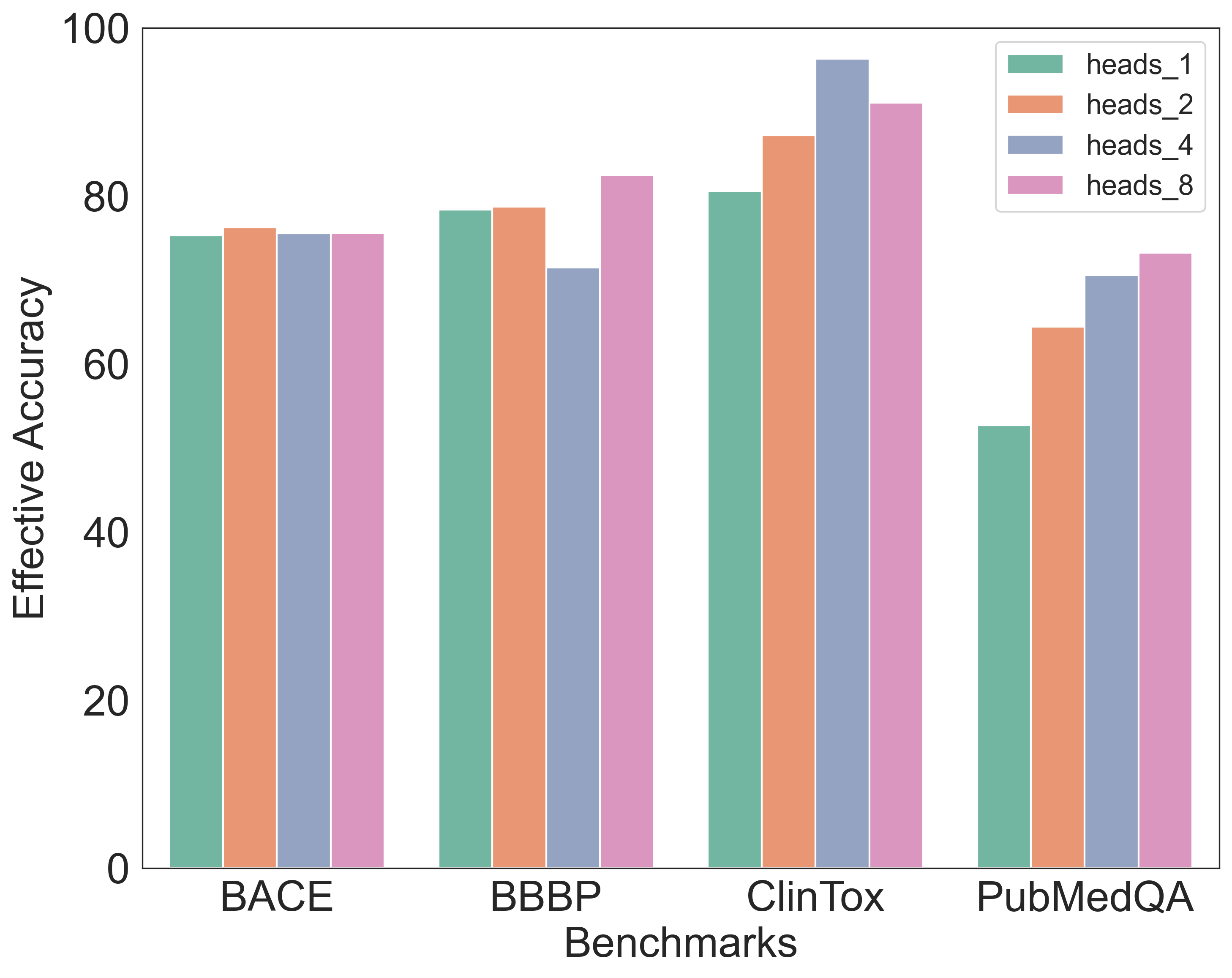}
        \caption{$K=5$, \textit{Qwen3-4B}}
        \label{fig:qwen4b_k_5}
    \end{subfigure}
    \begin{subfigure}{0.49\linewidth}
            \centering
        \includegraphics[width=\linewidth]{ 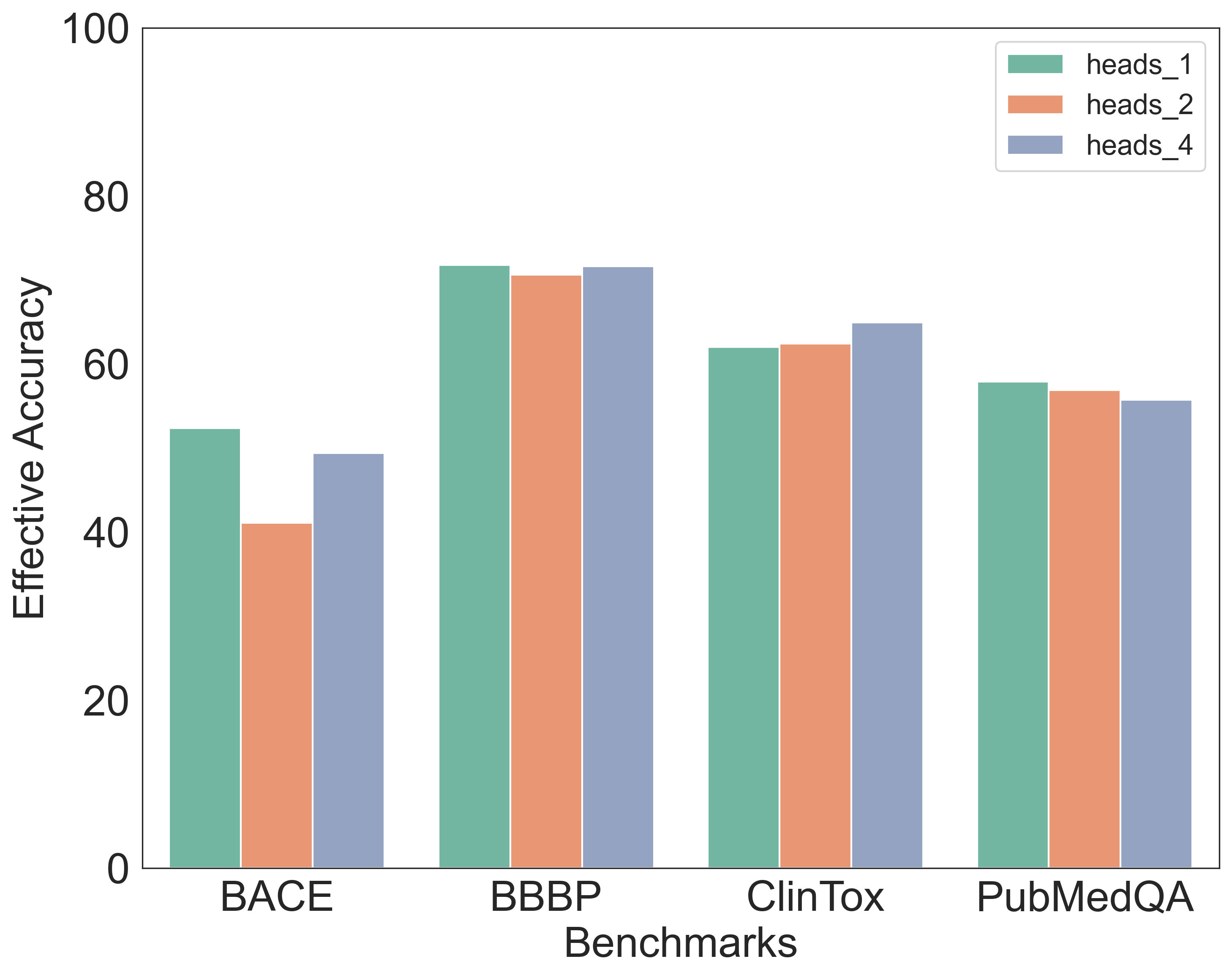}
        \caption{$K = 1$, \textit{Llama3.2-3B-Instruct}}
        \label{fig:llama3_k_1}
    \end{subfigure}
    \hfill
    \begin{subfigure}{0.49\linewidth}
        \centering
        \includegraphics[width=\linewidth]{ 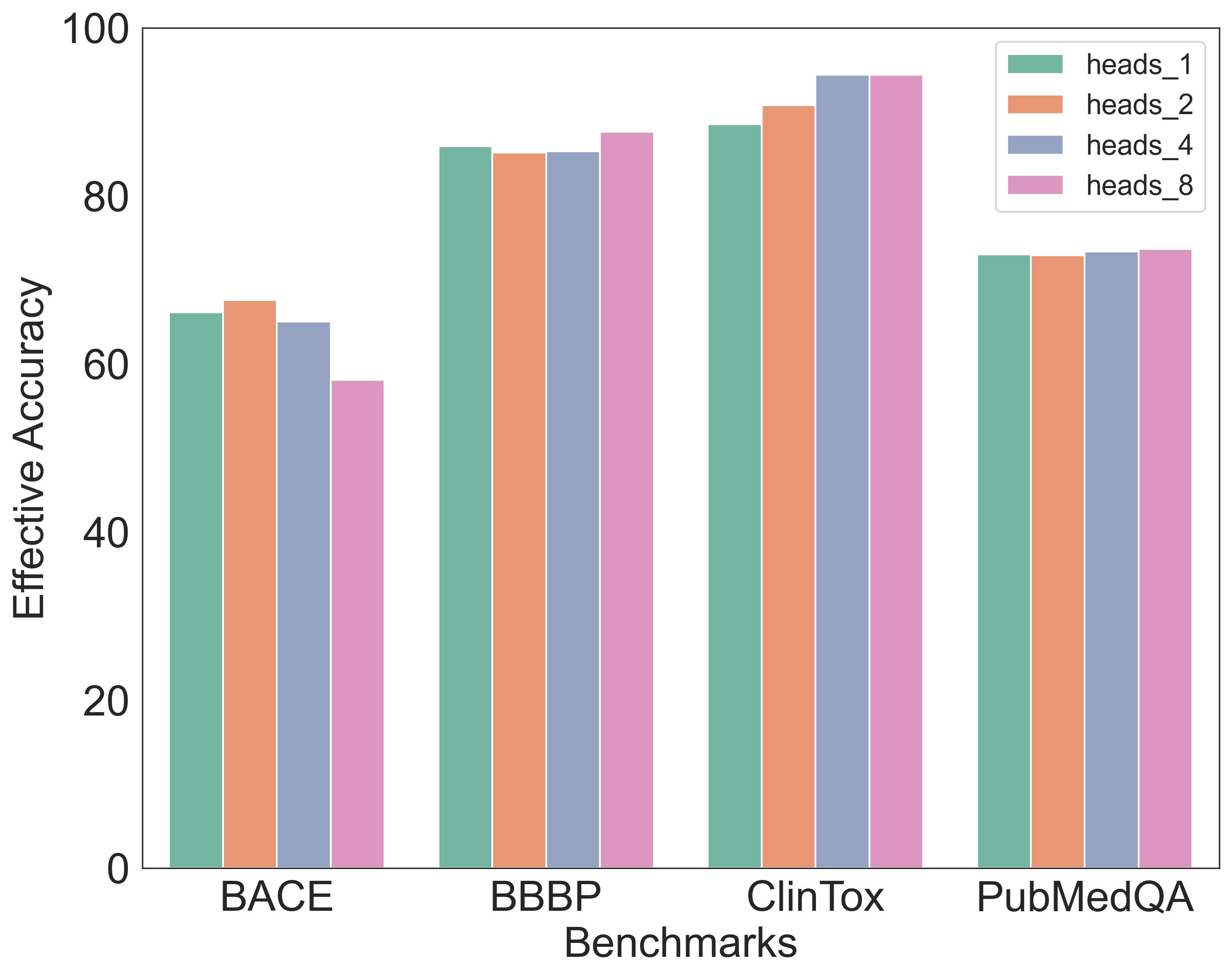}
        \caption{$K = 5$, \textit{Llama3.2-3B-Instruct}}
        \label{fig:llama3_k_5}
    \end{subfigure}
    \caption{Varying number of heads in \approach. Performance averaged across 3 seeded runs.}
    \label{fig:qwen4b_k_1_k_5}
\end{figure*}


We further study the impact of the number of heads. From Figure \ref{fig:qwen4b_k_1_k_5}, we observe that when $K=1$, increasing the number of heads does not improve effective accuracy, likely due to the limited context---only a single exemplar or document---offering little room for multiple heads to learn diverse representations. In contrast, at $K=5$, increasing the number of heads generally leads to higher effective accuracy across benchmarks, suggesting that the benefit of multiple heads emerges only when sufficient context is available.

\textit{Findings.} \approach can effectively exploit additional context with an increasing number of heads.


\section{Broader Impact}
\label{sec:broader_impact}
The significance of our work lies in its potential to reduce training and maintenance costs associated with hosting and adapting foundation models to multiple domains. Furthermore, our method enhances user privacy by enabling domain-specific customisation on the user end rather than the server end.


\newpage

\end{document}